\documentclass[twoside,11pt]{article}
\usepackage{jmlr2e}

\usepackage[utf8]{inputenc} 
\usepackage[T1]{fontenc}    
\usepackage{bbm}  
\usepackage[english]{babel}
\usepackage{hyperref}
\usepackage{amsmath}
\usepackage{mdframed} 
\usepackage{tablefootnote}
\usepackage{algorithm}
\usepackage{algorithmic}
\usepackage[titlenumbered,ruled,noend,algo2e]{algorithm2e}

\SetCommentSty{mycommfont}
\SetEndCharOfAlgoLine{}

\usepackage{xcolor}
\hypersetup{ hidelinks }
\definecolor{linkcolor}{RGB}{83,83,182}
\definecolor{citecolor}{RGB}{128,0,128}
\hypersetup{
    colorlinks=true,
    citecolor=citecolor,
    linkcolor=linkcolor,
    urlcolor=linkcolor
}

\usepackage[capitalise,nameinlink]{cleveref}
\usepackage{booktabs}
\usepackage{quentin_shortcuts}

\usepackage{chngcntr}
\usepackage{apptools}
\usepackage[group-separator={,},group-minimum-digits=3]{siunitx}

\usepackage{graphicx}
\usepackage{caption}
\usepackage{subcaption}

\usepackage{geometry}
\usepackage{lastpage}

\begin{document}
\title{Implicit Differentiation for Fast Hyperparameter Selection in Non-Smooth Convex Learning}
%
\jmlrheading{23}{2022}{1-\pageref{LastPage}}{5/21; Revised
3/22}{4/22}{21-0486}{Quentin Bertrand, Quentin Klopfenstein, Mathurin Massias, Mathieu Blondel, Samuel Vaiter, Alexandre Gramfort and Joseph Salmon}
\ShortHeadings{
    Implicit Differentiation in Non-Smooth Convex Learning}{
     Bertrand, Klopfenstein, Massias, Blondel, Vaiter, Gramfort and Salmon
    }
\firstpageno{1}
\author{\name Quentin Bertrand$^*$
       \email quentin.bertrand@inria.fr \\
       \addr Université Paris-Saclay, Inria, CEA, Palaiseau, France
       \AND
       \name Quentin Klopfenstein$^*$ \email quentin.klopfenstein@u-bourgogne.fr \\
       \addr Institut Mathématique de Bourgogne, Université de Bourgogne, Dijon, France
       \AND
       \name Mathurin Massias \email mathurin.massias@gmail.com \\
       \addr MaLGa, DIBRIS, Universit\`a degli Studi di Genova, Genova, Italy
       \AND
       \name Mathieu Blondel \email mblondel@google.com \\
       \addr Google Research, Brain team, Paris, France
       \AND
       \name Samuel Vaiter \email samuel.vaiter@math.cnrs.fr \\
       \addr CNRS and Institut Mathématique de Bourgogne, Université de Bourgogne, Dijon, France
       \AND
       \name Alexandre Gramfort \email alexandre.gramfort@inria.fr \\
       \addr Université Paris-Saclay, Inria, CEA, Palaiseau, France
       \AND
       \name Joseph Salmon \email joseph.salmon@umontpellier.fr \\
       \addr IMAG, Universit\'e de Montpellier, CNRS, Montpellier, France
       }

       \editor{Massimiliano Pontil}

       \maketitle
    \begin{abstract}%
%
Finding the optimal hyperparameters of a model can be cast as a bilevel optimization problem, typically solved using
zero-order techniques.
In this work we study first-order methods when the inner optimization problem is convex but non-smooth.
We show that the forward-mode differentiation of proximal gradient descent and proximal coordinate descent yield sequences of Jacobians converging toward the exact Jacobian.
Using \implicitfull, we show it is possible to leverage the non-smoothness of the inner problem to speed up the computation.
Finally, we provide a bound on the error made on the hypergradient when the inner optimization problem is solved approximately.
Results on regression and classification problems reveal computational benefits for hyperparameter optimization, especially when multiple hyperparameters are required.

    \end{abstract}
\medskip
%
%

\section{Introduction}
Almost all models in machine learning require at least one hyperparameter, the
tuning of which drastically affects accuracy.
This is the case for many popular estimators, where the regularization hyperparameter controls the trade-off between a data fidelity term and a regularization term.
Such estimators, including Ridge regression \citep{Hoerl_Kennard70}, Lasso
\citep{Tibshirani96,Chen_Donoho_Saunders98}, elastic net \citep{Zou_Hastie05},
sparse logistic regression \citep{Koh_Kim_Boyd07},
support-vector machine/SVM \citep{Boser_Guyon_Vapnik92,Platt99}
are often cast as an optimization problem (\Cref{table:inner})
\begin{table}[t]
  \centering
  \caption{Examples of non-smooth inner problems as in \eqref{pb:generic_inner_pb}.}
  \begin{tabular}{|c|c|c|c|}
    \hline
  Inner problem, $\Phi$
  & $f(\beta)$
  & $g_j(\beta_j, \lambda)$
  & $e^{\lambda_{\max}}$
  \\
  \hline
  \rule{0pt}{2.5ex}
  Lasso
  & $\frac{1}{2n} \|y - X\beta\|^2$
  & $e^{\lambda} |\beta_j|$
  & $\tfrac{1}{n}\normin{X^\top y}_{\infty}$
  \\[1mm]
  elastic net
  & $\frac{1}{2n} \|y - X\beta\|^2$
  & $e^{\lambda_1} |\beta_j| + \tfrac{1}{2}e^{\lambda_2}\beta_j^2$
  & $\tfrac{1}{n}\normin{X^\top y}_{\infty}$
  \\[1mm]
  sparse log. reg.
  & $\frac{1}{n} \sum_{i=1}^{n} \ln (1 + e^{-y_i X_{i:} \beta})$
  & $e^{\lambda} |\beta_j|$
  & $\tfrac{1}{2n}\normin{X^\top y}_{\infty}$
  \\[1mm]
  dual SVM
  & $\frac{1}{2} \normin{(y \odot X)^\top \beta}^2 - \sum_{j=1}^p \beta_j$
  & $\iota_{[0, e^{\lambda}]}(\beta_j)$
  & $-$
  \\
  \hline
  \end{tabular}
  \label{table:inner}
  \end{table}
\begin{align}\label[pb_multiline]{pb:generic_inner_pb}
    \hat \beta^{(\lambda)} \in
    \argmin_{\beta \in \bbR^p}
    \Phi(\beta, \lambda)
    \eqdef
    f(\beta)
    + \underbrace{\sum_{j=1}^p g_j(\beta_j, \lambda)}_{\eqdef g(\beta, \lambda)}
    \enspace,
\end{align}
with smooth $f: \bbR^p \rightarrow \bbR$ (\ie with Lipschitz gradient), proper closed convex (possibly non-smooth) functions $g_j(\cdot, \lambda)$, and a regularization hyperparameter $\lambda \in \bbR^r$.
In the examples of \Cref{table:inner}, the computation of $f$ involves a design matrix $X \in \bbR^{n \times p}$; and the cost of computing $\nabla f(\beta)$ is $\cO(np)$.
In the SVM example, since we consider the dual problem, we chose to reverse the roles of $n$ and $p$ to enforce $\beta \in \bbR^p$.
We often drop the $\lambda$ dependency and write $\hat\beta$ instead of $\hat\beta^{(\lambda)}$ when it is clear from context.

For a fixed $\lambda$, the issue of solving efficiently \Cref{pb:generic_inner_pb} has been largely explored.
If the functions $g_j$ are smooth, one can use solvers such as L-BFGS \citep{Liu_Nocedal1989}, SVRG \citep{Johnson_Zhang2013,Zhang_Mahdavi_Jin2013}, or SAGA \citep{Defazio_Bach_LacosteJulien2014}.
When the functions $g_j$ are non-smooth, \Cref{pb:generic_inner_pb} can be tackled efficiently with stochastic algorithms \citep{Pedregosa_Leblond_LacosteJulien2017} or using working set methods \citep{Fan_Lv2008,Tibshirani_Bien_Friedman_Hastie_Simon_Tibshirani12}
combined with coordinate descent
\citep{Tseng_Yun09}, see overview by \citet{Massias_Vaiter_Gramfort_Salmon20}.
The question of \emph{model selection}, \ie how to select the hyperparameter $\lambda \in \bbR^r$ (potentially multidimensional), is more open, especially when the dimension $r$ of the regularization hyperparameter $\lambda$ is large.

For the Lasso, a broad literature has been devoted to parameter tuning.
Under strong hypothesis on the design matrix $X$, it is possible to derive guidelines for the setting of the regularization parameter $\lambda$ \citep{Lounici08,Bickel_Ritov_Tsybakov09,Belloni_Chernozhukov_Wang11}.
Unfortunately, these guidelines rely on quantities which are typically unknown in practice, and Lasso users still have to resort to other techniques to select the hyperparameter $\lambda$.

A popular approach for hyperparameter selection is \emph{hyperparameter optimization} \citep{Kohavi_John1995,Hutter_Lucke_SchmidtThieme2015,Feurer_Hutter2019}: one selects the hyperparameter $\lambda$
such that the regression coefficients $\hat \beta^{(\lambda)}$ minimize a given criterion $\cC: \bbR^p \rightarrow \bbR$.
Here $\cC$ should ensure good generalization, or avoid overcomplex models.
Common examples (see \Cref{table:criterion}) include the hold-out loss \citep{Devroye_Wagner1979}, the cross-validation loss (CV,
\citealt{Stone_Ramer65}, see \citealt{Arlot_Celisse10} for a survey), the AIC \citep{Akaike74}, BIC \citep{Schwarz78} or SURE \citep{Stein81} criteria.
Formally, the hyperparameter optimization problem is a bilevel optimization problem \citep{Colson_Marcotte_Savard2007}
\begin{align}\label[pb_multiline]{pb:bilevel_opt}
  \begin{aligned}
    &\argmin_{\lambda \in \bbR^r}
    \left\{
    \cL(\lambda) \eqdef
    \cC \left (\hat \beta^{(\lambda)}  \right)
    \right\}
    \\
    &\st \hat \beta^{(\lambda)} \in \argmin_{\beta \in \bbR^p}
    \Phi(\beta, \lambda)
     \enspace.
\end{aligned}
\end{align}
%
%
\begin{table}[t]
\centering
\caption{Examples of outer criteria used for hyperparameter selection.}
\scalebox{0.90}{
\begin{tabular}{|c|c|c|}
\hline
Criterion & Problem type   & Criterion $\cC(\beta)$ \\
\hline
\rule{0pt}{2.5ex}
Hold-out mean squared error & Regression
& $\frac{1}{n} \| y^{\text{val}} - X^{\text{val}}\beta\|^2$ \\[1mm]
Stein unbiased risk estimate (SURE)\tablefootnote{
  For a linear model $y = X\beta + \varepsilon$,
  with $\varepsilon \sim  \cN(0, \sigma^2)$,
  the degree of freedom ($\text{dof}$, \citealt{Efron86}) is defined as $\text{dof}(\beta) = \sum_{i=1}^n \text{cov}(y_i, (X\beta)_i) / \sigma^2$.}
& Regression
& $\|y - X\beta\|^2 - n \sigma^2 + 2\sigma^2 \text{dof}(\beta)$ \\[1mm]
Hold-out logistic loss
& Classification
& $\frac{1}{n}\sum_{i=1}^{n} \ln(1 + e^{-y^{\text{val}}_i X_{i :}^{\text{val}} \beta})$
\\[1mm]
Hold-out smoothed Hinge loss\tablefootnote{The smoothed Hinge loss is given by
$
\ell(x)
= \frac{1}{2} - x$ if $x \leq 0$, $\frac{1}{2} (1 - x)^2$ if $0 \leq x \leq 1$ and 0 else.  }
& Classification
& $\frac{1}{n} \sum_{i=1}^{n} \ell (y^{\text{val}}_i, X_{i :}^{\text{val}} \beta)$\\[1mm]
\hline
\end{tabular}
}
\label{table:criterion}
\end{table}
Popular approaches to solve (the generally non-convex) \Cref{pb:bilevel_opt} include zero-order optimization (gradient-free) techniques such as
grid-search,
random-search \citep{Rastrigin63,Bergstra_Bengio12,Bergstra13}
or Sequential Model-Based Global Optimization (SMBO), often referred to as Bayesian optimization \citep{Mockus1989,Jones_Schonlau_Welch1998,Forrester_Sobester_Keane2008,Brochu_Cora_deFreitas10,Snoek_Larochelle_Ryan12}.
Grid-search is a naive discretization of \Cref{pb:bilevel_opt}. It consists in evaluating the outer function $\cL$ on a grid of hyperparameters, solving one inner optimization \Cref{pb:generic_inner_pb} for each $\lambda$ in the grid (see \Cref{fig:intro_cv}).
For each inner problem solution $\hat \beta^{(\lambda)}$,
the criterion $\cC(\hat \beta^{(\lambda)})$ is evaluated,
and the model achieving the lowest value is selected.
Random-search has a similar flavor, but one randomly selects where the criterion must be evaluated.
Finally, SMBO models the objective function $\cL$ via a function amenable to uncertainty estimates on its predictions such as a Gaussian process.
Hyperparameter values are chosen iteratively to maximize a function such as the
expected improvement as described, \eg by \citet{BergstraBardenetBengioKegl2011}.
However, these zero-order methods share a common drawback: they scale exponentially with the dimension of the search space \citep[Sec. 1.1.2]{Nesterov04}.

When the hyperparameter space is continuous and the regularization path $\lambda \mapsto \hat \beta^{(\lambda)}$ is well-defined and almost everywhere differentiable, first-order optimization methods are well suited to solve the bilevel optimization \Cref{pb:bilevel_opt}.
Using the chain rule, the gradient of $\cL$ \wrt $\lambda$, also referred to as the
\emph{hypergradient}, evaluates to
\begin{align}\label{eq:hypergrad}
    \nabla_{\lambda}\cL(\lambda)
    & =
    \hat \jac^{\top}_{(\lambda)}\nabla \cC(\hat \beta^{(\lambda)}) \enspace,
\end{align}
with $\hat \jac_{(\lambda)} \in \bbR^{p \times r}$ the \emph{Jacobian} of the function $\lambda \mapsto \hat \beta^{(\lambda)}$,
\begin{align}
    \label{eq:jac}
    \hat \jac_{(\lambda)} \eqdef
    \begin{pmatrix}
    &\tfrac{\partial \hat \beta^{(\lambda)}_1}{\partial \lambda_1} & \hdots
    &\tfrac{\partial \hat \beta^{(\lambda)}_1}{\partial \lambda_r}\\
    &\vdots &\hdots &\vdots\\
    &\tfrac{\partial \hat \beta^{(\lambda)}_p}{\partial \lambda_1}&\hdots
    &\tfrac{\partial \hat \beta^{(\lambda)}_p}{\partial \lambda_r}
    \end{pmatrix}\enspace .
\end{align}


An important challenge of applying first-order methods to solve \Cref{pb:bilevel_opt} is evaluating the hypergradient in \Cref{eq:hypergrad}.
There are three main algorithms to compute the hypergradient $\nabla_{\lambda}\cL(\lambda)$: implicit differentiation
\citep{Larsen_Hansen_Svarer_Ohlsson96,Bengio00} and automatic differentiation
using the \backward \citep{Linnainmaa1970,Lecun_Bottou_Orr_Muller2012}
or the \forward
\citep{Wengert1964,Deledalle_Vaiter_Fadili_Peyre14,Franceschi_Donini_Frasconi_Pontil17}.
As illustrated in \Cref{fig:intro_cv}, once the hypergradient
in \Cref{eq:hypergrad} has been computed, one can solve \Cref{pb:bilevel_opt}
with first-order schemes, \eg gradient descent.

\begin{figure}[t]
  \centering
  \begin{subfigure}[b]{1\textwidth}
      \includegraphics[width=1\linewidth]{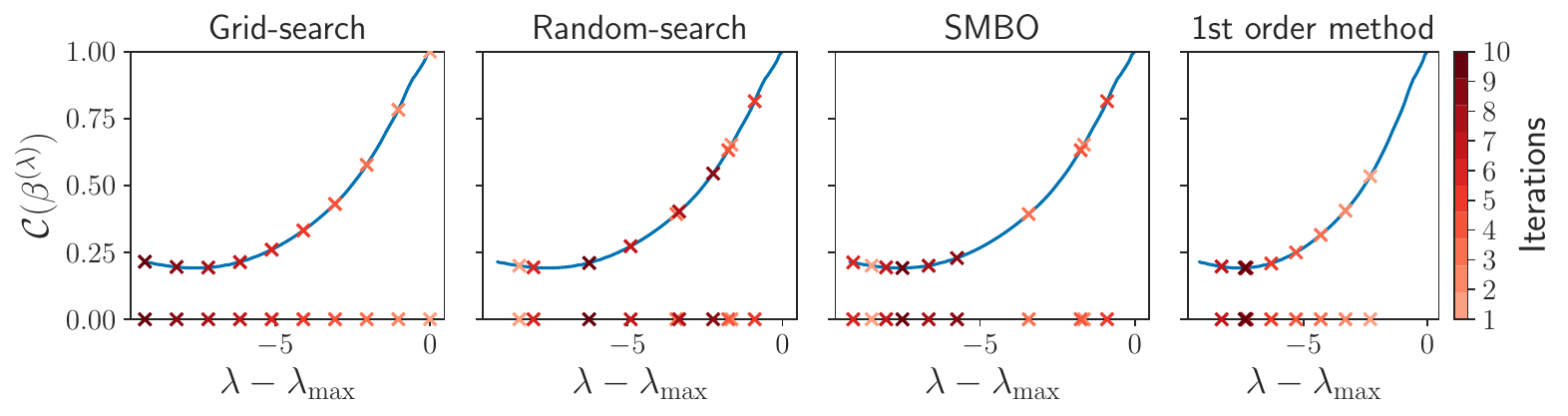}
      \includegraphics[width=1\linewidth]{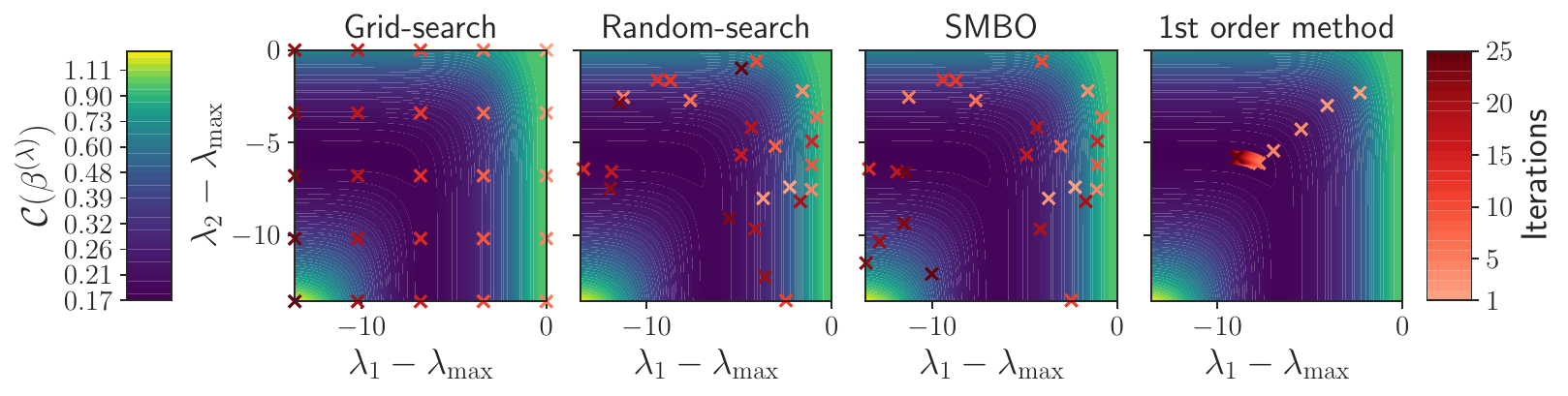}
  \end{subfigure}
\caption{\emph{5-fold cross-validation error} $\cC(\beta^{(\lambda)})$:
(top) Lasso CV error \wrt $\lambda$ for multiple hyperparameter optimization
methods on the \emph{real-sim} data set, and (bottom) elastic net CV error \wrt
$\lambda_1$ and $\lambda_2$ on the \emph{rcv1} data set.
Crosses represent the $10$ (top) or $25$ (bottom) first error evaluations for each method.
}
\label{fig:intro_cv}
\end{figure}

\textit{Contributions.}
We are interested in tackling the bilevel optimization \Cref{pb:bilevel_opt}, with a non-smooth inner optimization \Cref{pb:generic_inner_pb}. More precisely,
\begin{itemize}
  \item We show that classical algorithms used to compute hypergradients for smooth inner problem have theoretically grounded non-smooth counterparts.
    We provide in \Cref{thm:non-smooth_implicitdiff} an \implicitfull formula for non-smooth optimization problems.
    We obtain in \Cref{thm:approximate_gradient},
    for the first time in the non-smooth case,
    error bounds with respect to the hypergradient when the inner
    problem and the linear system involved are only solved approximately.
    We obtain in \Cref{thm:iterdiff_linear_convergence} convergence rates on the hypergradient for iterative differentiation of non-smooth optimization problems.
  \item Based on the former contributions we propose an algorithm to tackle \Cref{pb:bilevel_opt}.
    We develop an efficient implicit differentiation algorithm to compute the hypergradient in \Cref{eq:hypergrad}, leveraging the sparsity of the Jacobian and enabling the use of state-of-the-art solvers (\Cref{alg:implicit}).
    We combine in \Cref{alg:bilevel_approx} this fast hypergradient computation with a gradient descent
    scheme to solve \Cref{pb:bilevel_opt}.
  \item We provide extensive experiments on diverse data sets and estimators (\Cref{sec:experiments}).
  We first show that implicit differentiation significantly outperforms other hypergradient methods (\Cref{sub:expe_hypergrad_computation}).
  Then, leveraging sparsity, we illustrate computational benefits of first-order optimization \wrt zero-order techniques for solving \Cref{pb:bilevel_opt} on Lasso, elastic net and multiclass logistic regression (\Cref{sub:expe_bilevel_optimization}).
  \item We release our implementation as a high-quality, documented and tested Python package: \url{https://github.com/qb3/sparse-ho}.
\end{itemize}

\textit{General notation.}
We write $\normin{\cdot}$ the Euclidean norm on vectors.
For a set $S$, we denote by $S^c$ its complement.
We denote $[p]=\{1,\dots,p\}$.
We denote by $(e_j)_{j=1}^{p}$ the vectors of the canonical basis of $\mathbb{R}^p$.
We denote the coordinate-wise multiplication of two vectors $u$ and $v$ by $u \odot v$, and by $u \odot M$ the row-wise multiplication between a vector and a matrix.
The $i$-th line of the matrix $M$ is $M_{i:}$ and its $j$-th column is $M_{:j}$.
The spectral radius of a matrix $M\in \mathbb{R}^{n\times n}$ is denoted $\rho(M) = \max_i |s_i|$ where $s_1, \hdots, s_n$ are the eigenvalues of $M$.
For a matrix $M$, we write that $M \succ 0$ if $M$ is positive definite.
The regularization parameter, possibly multivariate, is denoted by $\lambda = (\lambda_1,\dots,\lambda_r)^\top\in\bbR^r$.
Recall that for a locally integrable function $f: x \in \Omega \mapsto \mathbb{R}$, where $\Omega$ is an open subset of $\mathbb{R}^n$, its weak partial derivative  \citep{evan1992measure} with respect to $x_i$ in $\Omega$ is the locally integrable function $g_i$ on $\Omega$ such that
\begin{align}
  \int_\Omega g_i(x) \phi(x) dx
  =
  -\int_\Omega f(x) \frac{\partial \phi(x)}{\partial x_i} dx
  \enspace ,
\end{align}
holds for all functions $\phi$ that are continuously differentiable and of compact support.
For vector valued function $f$, we denote by $\jac (x)$ its weak Jacobian \ie the matrix composed of weak partial derivatives.
An important use of this notation in the rest of the paper is  $\hat \jac_{(\lambda)} \eqdef (\nabla_\lambda \hat \beta_1^{(\lambda)}, \dots, \nabla_\lambda \hat\beta_p^{(\lambda)})^\top \in \bbR^{p \times r}$ the weak Jacobian of $\hat \beta^{(\lambda)}$ \wrt $\lambda$.
%
\textit{Convex analysis.}
For a convex function $h: \bbR^p \to \bbR$, the proximal operator of $h$ is defined, for any $x\in\bbR^p$, as:
  $\prox_{h}(x)
  =
  \argmin_{y\in \mathbb{R}^{p}} \frac{1}{2}
  \|x - y \|^{2} + h(y)$.
The subdifferential of $h$ at $x$ is denoted $\partial h (x) = \condset{u \in \bbR^p}{\forall z \in \bbR^p, h(z) \geq h(x) + u^\top (z - x)}$.
A function is said to be \emph{smooth} if it has Lipschitz gradients.
Let $f$ be a $L$-smooth function.
Lipschitz constants of the functions $\nabla_j f$ are denoted by $L_j$; hence for all $x \in \bbR^p$, $h \in \bbR$
\begin{align}
    |\nabla_j f(x + h e_j) - \nabla_j f(x)|
    \leq
    L_j |h| \nonumber
    \enspace .
\end{align}
For a function $f$, its gradient restricted to the indices in a set $S$ is denoted $\nabla_S f$.
For a set $\Xi \subset \mathbb{R}^{p}$, its relative interior is noted $\ri (\Xi)$, and
its indicator function is defined for any $x\in \bbR^p$ by
$ \iota_{\Xi}(x) = 0$ if $x \in \Xi$ and $ +\infty$ otherwise.
A function
$h:\mathbb{R} \rightarrow \mathbb{R} \cup \{ +\infty\}$
is said to be proper if
$\dom(h) = \{x \in \bbR : h(x) < +\infty\} \neq \emptyset$),
and closed if for any $\alpha\in \bbR$, the sublevel set $\{x\in \dom(h): h(x)\leq \alpha \}$ is a closed set.

For a function $\psi : \mathbb{R}^p \times \bbR^r \mapsto \bbR^p$, we denote $\partial_z \psi$ the weak Jacobian \wrt the first variable and $\partial_{\lambda}\psi$ the weak Jacobian \wrt the second variable.
The proximal operator of $g(\cdot, \lambda)$ can be seen as such a function $\psi$ of $\beta$ and $\lambda$ (see \Cref{table:inner} for examples)
\begin{align*}
  \bbR^p \times \bbR^r & \rightarrow \bbR^p \\
  (z,  \lambda) &\mapsto \prox_{g(\cdot, \lambda)}(z)
  = \psi(z, \lambda)
  \enspace .
\end{align*}
In this case we denote
$\partial_z  \prox_{g(\cdot, \lambda)} \eqdef \partial_z \psi$ and $\partial_{\lambda}  \prox_{g(\cdot, \lambda)} \eqdef \partial_{\lambda} \psi$.
Since we consider only separable penalties $g(\cdot, \lambda)$, $\partial_z  \prox_{g(\cdot, \lambda)}$
is a diagonal matrix, so to make notation lighter, we write $\partial_z  \prox_{g(\cdot, \lambda)}$ for its diagonal.
We thus have
\begin{align*}
  \partial_z  \prox_{g(\cdot, \lambda)}
  = (\partial_z  \prox_{g_j(\cdot, \lambda)})_{j \in [p]} & \in \bbR^p \quad (\text{by separability of $g$})\\
  \partial_{\lambda}  \prox_{g(\cdot, \lambda)} &\in \bbR^{p \times r}
  \enspace .
\end{align*}
Explicit partial derivatives formulas for usual proximal operators can be found in \Cref{table:partial_derivatives_prox}.

\begin{table}[t]
  \centering
  \caption{Partial derivatives of proximal operators used.
  }
  \scalebox{0.90}{
  \begin{tabular}{|c|c|c|c|}
    \hline
    $g_j(\beta_j, \lambda)$
    & $\prox_{g_j(\cdot, \lambda)}(z_j)$
    & $\partial_z \prox_{g_j(\cdot, \lambda)}(z_j) $
    & $\partial_{\lambda} \prox_{g_j(\cdot, \lambda)}(z_j) $ \\
  \hline
  \rule{0pt}{2.5ex}
  $e^{\lambda} \beta_j^2 / 2$
  & $z_j / (1 + e^\lambda)$
  & $1 / (1 + e^\lambda)$
  & $-z_j e^\lambda / (1 + e^\lambda)^2$
  \\[1mm]
  $e^{\lambda} |\beta_j|$
  & $\ST(z_j, e^\lambda)$
  & $|\sign(\ST(z_j, e^\lambda))|$
  & $- e^\lambda \sign(\ST(z_j, e^\lambda))$
  \\[1mm]
  $e^{\lambda_1} |\beta_j| + \tfrac{1}{2}e^{\lambda_2} \beta_j^2$
  & $\frac{\ST(z_j, e^{\lambda_1})}{ 1 + e^{\lambda_2}} $
  & $\frac{|\sign(\ST(z_j, e^{\lambda_1}))| }{1 + e^{\lambda_2}} $
  & $ \left (
      \frac{-e^{\lambda_1} \sign(\ST(z_j, e^{\lambda_1}))}{1 + e^{\lambda_2}},
      \frac{-\ST(z_j, e^{\lambda_1}) e^{\lambda_2}}{(1 + e^{\lambda_2})^2}
      \right )$
  \\[1mm]
  $\iota_{[0, e^{\lambda}]}(\beta_j)$
  & $\max(0, \min(z_j, e^{\lambda}))$
  & $\ind_{]0,  e^{\lambda}[}(z_j)  $
  & $e^{\lambda} \ind_{z_j > e^{\lambda}}  $
  \\
  \hline
  \end{tabular}
  }
  \label{table:partial_derivatives_prox}
\end{table}

%

\section{Related Work}
\label{sec:hypergrad_smooth}
The main challenge to evaluate the hypergradient $\nabla_{\lambda}\mathcal{L}(\lambda)$ is the computation of the Jacobian $\jac_{(\lambda)}$.
We first focus on the case where $ \Phi(\cdot, \lambda)$ is convex and smooth for any $\lambda$.
%

\textit{Implicit differentiation.}
We recall how the \emph{implicit differentiation}\footnote{Note that \emph{implicit} refers to the implicit function theorem, but leads to an \emph{explicit} formula for the gradient.} formula of the gradient $\nabla_{\lambda} \cL(\lambda)$ is obtained for smooth inner optimization problems.
We will provide a generalization to non-smooth optimization problems in \Cref{sub:implicit_nonsmooth}.
\begin{theorem}
    \emph{\citep{Bengio00}.}
    \label{thm:smooth_implicitdiff}
    Let $\hat{\beta}^{(\lambda)} \in \argmin_{\beta \in \bbR^p} \Phi(\beta, \lambda)
    $ be a solution of \Cref{pb:generic_inner_pb}.
    Assume that for all $\lambda>0$,
    $\Phi(\cdot, \lambda)$ is a convex smooth function,
    $\nabla_{\beta}^2 \Phi(\hat \beta^{(\lambda)}, \lambda) \succ 0$,
    and that for all $\beta \in \bbR^p$,  $\Phi(\beta, \cdot)$ is differentiable over $]0,+\infty[$.
    Then the hypergradient $\nabla_{\lambda} \cL(\lambda)$ reads
    \begin{equation} \label{eq:grad_smooth}
        \underbrace{\nabla_{\lambda}\mathcal{L}(\lambda)}_{\in \bbR^r}
        =
        \underbrace{- \nabla_{\beta, \lambda}^2
        \Phi (\hat \beta^{(\lambda)}, \lambda )}_{\in \bbR^{r \times p}}
        {\underbrace{\left ( \nabla_{\beta}^2 \Phi(\hat \beta^{(\lambda)}, \lambda) \right )}_{\in \bbR^{p \times p}}}^{-1}
        \underbrace{\nabla \mathcal{C}(\hat \beta^{(\lambda)})}_{\in \bbR^p}
        \enspace .
    \end{equation}

\end{theorem}
\begin{proof}
    For a smooth convex function $ \beta \mapsto \Phi(\beta, \lambda)$ the first-order condition writes:
    \begin{align}\label{eq:fixed_point_smooth}
        \nabla_{\beta}\Phi(\hat \beta^{(\lambda)}, \lambda) = 0 \enspace ,
    \end{align}
    for any $\hat \beta^{(\lambda)}$ solution of the inner problem.
    Moreover, if $\lambda \mapsto \nabla_{\beta}\Phi(\hat \beta^{(\lambda)}, \lambda)$ is differentiable, differentiating \Cref{eq:fixed_point_smooth} \wrt $\lambda$ leads to
    \begin{align}
        \label{eq:grad}
        \nabla_{\beta, \lambda}^{2} \Phi(\hat \beta^{(\lambda)},\lambda)
        + \hat \jac^{\top}_{(\lambda)}\nabla_{\beta}^{2}\Phi(\hat \beta^{(\lambda)}, \lambda)
        = 0
        \enspace .
    \end{align}
    The Jacobian $ \hat \jac^{\top}_{(\lambda)}$ is computed by solving the following linear system
    \begin{align}
        \hat \jac_{(\lambda)}^\top
        = - \underbrace{
            \nabla_{\beta, \lambda}^2 \Phi( \hat \beta^{(\lambda)}, \lambda )
        }_{\in \bbR^{r \times p}}
        {\underbrace{\left ( \nabla_{\beta}^2
        \Phi(\hat \beta^{(\lambda)}, \lambda) \right )}_{\in \bbR^{p \times p}}}^{-1}
        \enspace.
        \label{eq:jac}
    \end{align}
    Plugging \Cref{eq:jac} into \Cref{eq:hypergrad} yields the desired result.
\end{proof}

The computation of the gradient via implicit differentiation (\Cref{eq:grad_smooth}) involves the resolution of a $p\times p$ linear system \citep[Sec. 4]{Bengio00}.
This potentially large linear system can be solved using different algorithms such as conjugate gradient (\citealt{Hestenes_Stiefel1952}, as in \citealt{Pedregosa16}) or fixed point methods (\citealt{Lions_Mercier1979,Tseng_Yun09}, as in \citealt{Grazzi_Franceschi_Pontil_Salzo2020}).
Implicit differentiation has been used for model selection of multiple estimators with smooth regularization term: kernel-based models \citep{Chapelle_Vapnick_Bousquet_Mukherjee02,Seeger08}, weighted Ridge estimator \citep{Foo_Do_Ng08}, neural networks \citep{Lorraine_Vicol_Duvenaud2019} or meta-learning \citep{Rajeswaran_Finn_kakade_Levine2019}.
In addition to hyperparameter selection, it has been applied successfully in natural language processing \citep{Bai_Kolter_Koltun2019} and computer vision \citep{Bai_Koltun_Kolter2020}.

\Cref{pb:generic_inner_pb} is typically solved using iterative solvers.
In practice, the number of iterations is limited to reduce computation time, and also since
very precise solutions are generally not necessary for machine learning tasks.
Thus, \Cref{eq:fixed_point_smooth} is not exactly satisfied at machine precision, and consequently the linear system to solve \Cref{eq:grad_smooth} does not lead to the exact gradient $\nabla_{\lambda}\mathcal{L}(\lambda)$, see \citet{Ablin_Peyre_Moreau2020} for  quantitative convergence results.
However, \citet{Pedregosa16} showed that one can resort to \emph{approximate gradients} when the inner problem is smooth, justifying that implicit differentiation can be applied using an approximation of $\hat \beta$.
Interestingly, this approximation scheme was shown to yield significant practical speedups when solving \Cref{pb:bilevel_opt}, while preserving theoretical properties of convergence toward the optimum.
Practitioners now have access to powerful software to use \implicitfull with smooth inner optimization problems \citep{jaxopt}.
%
%
%

\textit{Iterative differentiation.}
\label{sub:iterdiff_smooth}
%
Iterative differentiation computes the gradient $\nabla_{\lambda} \cL (\lambda)$
by differentiating through the iterates of the algorithm used to solve \Cref{pb:generic_inner_pb}.
Iterative differentiation can be applied using the \forward (\citealt{Wengert1964}) or the \backward (\citealt{Linnainmaa1970}).
Both rely on the chain rule, the gradient being decomposed as a large product of matrices, computed either in a forward or backward way.
Note that \forwardandbackward are algorithm-dependent: in this section we illustrate iterative differentiation for proximal gradient descent (PGD, \citealt{Lions_Mercier1979,Combettes_Wajs05}), using the \forward (\Cref{alg:forward_pgd}), and the \backward (\Cref{alg:backward_pgd}).

The most popular method in automatic differentiation is the \backward,
a cornerstone of deep learning \citep[Chap. 8]{Goodfellow_Courville_Bengio2016}.
Iterative differentiation for hyperparameter optimization can be traced back to \citet{Domke12}, who derived (for smooth loss functions) a \backward with gradient descent, heavy ball and L-BFGS algorithms.
It first computes the solution of the optimization \Cref{pb:generic_inner_pb} using an iterative solver, but requires storing the iterates along the computation for a backward evaluation of the hypergradient (\Cref{alg:backward_pgd}).
\citet{MacLaurin_Duvenaud_Adams15} used the \backward on stochastic gradient descent to select thousands of hyperparameters.
Alternatively, the \forward computes jointly the solution along with the gradient $\nabla_{\lambda} \cL (\lambda)$.
The \forward has been applied to hyperparameter optimization with smooth inner problems by  \citet{Franceschi_Donini_Frasconi_Pontil17}.
\citet{Deledalle_Vaiter_Fadili_Peyre14} paved the way for applying it to non-smooth optimization problems.
The \forward is memory efficient (no iterates storage) but more computationally expensive when the number of hyperparameters ($r$) is large;
see \citet{Baydin_Pearlmutter_Radul_Siskind18} for a survey.

\textit{Resolution of the bilevel \Cref{pb:bilevel_opt}.}
From a theoretical point of view, solving \Cref{pb:bilevel_opt} using gradient-based methods is also challenging, and results in the literature are quite scarce.
\citet{Kunisch_Pock13} studied the convergence of a semi-Newton algorithm
where both the outer and inner problems are smooth.
\citet{Franceschi_Frasconi_Salzo_Pontil18} gave similar results with weaker assumptions to unify hyperparameter optimization and meta-learning with a bilevel point of view.
They required the inner problem to have a unique solution for all $\lambda>0$ but do not have second-order assumptions on $\Phi$.
Recent results \citep{Ghadimi_Wang2018,Ji_Yang_Liang2020} have provided quantitative convergence toward a global solution of \Cref{pb:bilevel_opt}, under the assumption that the inner problem is strongly convex and one has the exact knowledge of the hypergradient Lipschitz constant.
%
%
\begin{figure}[tb]
    \begin{minipage}[t]{0.48\linewidth}
        {\fontsize{5}{4}\selectfont
            \begin{algorithm}[H]
            \SetKwInOut{Input}{input}
            \SetKwInOut{Init}{init}
            \SetKwInOut{Parameter}{param}
            \caption{\textsc{\Forward PGD}}
            \Input{
                $
                \lambda \in \bbR^r,
                \gamma > 0,
                n_{\mathrm{iter}} \in \bbN$, $\beta^{(0)} \in \bbR^p$, $\jac^{(0)} \in \bbR^{p \times r}$
                }
            \tcp{jointly compute coef. \& Jacobian }

                \For{$k = 1,\dots, n_{\mathrm{iter}}$}{
                    \tcp*[h]{update the regression coefficients}\\
                    $z^{(k)} =
                    \beta^{(k-1)} - \gamma \nabla f(\beta^{(k-1)})$
                    \tcp*[r]{GD step}

                    $\diff z^{(k)} =
                    \jac^{(k-1)} -
                    \gamma \nabla^2 f(\beta^{(k-1)}) \jac^{(k-1)} $

                    $\beta^{(k)} = \prox_{\gamma g(\cdot, \lambda)}(z^{(k)})$
                    \tcp*[r]{prox. step}

                    \tcp*[h]{update the Jacobian}\\
                    $\jac^{(k)} =
                    \partial_z \prox_{\gamma g(\cdot, \lambda)}(z^{(k)})
                    \odot \diff z^{(k)}$

                    $\jac^{(k)} \pluseq
                    \partial_{\lambda} \prox_{\gamma g(\cdot, \lambda)}(z^{(k)})$
                    \tcp*[r]{$\bigo(pr)$}
                }
            $v = \nabla \cC (\beta^{n_\mathrm{iter}})$

            \Return{$\beta^{n_\mathrm{iter}}, \jac^{n_\mathrm{iter} \top} v$}
            \label{alg:forward_pgd}
            \end{algorithm}
            }
    \end{minipage}\hfill
    \hfill
    \begin{minipage}[t]{0.48\linewidth}
        {\fontsize{5}{4}\selectfont
        \begin{algorithm}[H]
        \SetKwInOut{Input}{input}
        \SetKwInOut{Init}{init}
        \SetKwInOut{Parameter}{param}
        \caption{\textsc{\Backward PGD}}
        \Input{\hspace{-1mm}
        $
            \lambda \in \bbR^r,
            \gamma > 0,
            n_{\mathrm{iter}} \in \bbN$, $\beta^{(0)} \in \bbR^p$
            }
        \tcp{computation of $\hbeta$}

        \For{$k = 1,\dots, n_{\mathrm{iter}} $}{

            $z^{(k)} \hspace{-0.25em}
            = \hspace{-0.25em}
            \beta^{(k-1)}  -\gamma  \nabla f(\beta^{(k-1)})$
            \tcp*[r]{GD step}

            $\beta^{(k)}  = \prox_{\gamma g(\cdot, \lambda)} \left ( z^{(k)} \right )$
            \tcp*[r]{prox. step}

            }

        \tcp{backward computation of the gradient $g$}

        $v = \nabla \cC (\beta^{ (n_{\mathrm{iter}} ) })$,
        $h = 0_{\bbR^r}$

        \For{$k = n_{\mathrm{iter}}, n_{\mathrm{iter}}-1,\dots, 1$}{
            $h \pluseq v^\top
            \partial_{\lambda}\prox_{\gamma g(\cdot, \lambda)}(z^{(k)})$
            \tcp*[r]{$\bigo(pr)$}

            $v \leftarrow
            \partial_z \prox_{\gamma g(\cdot, \lambda)}(z^{(k)})\odot v$
            \tcp*[r]{$\bigo(p)$}

            $v \leftarrow (\Id - \gamma \nabla^2 f(\beta^{(k)}))v $
            \tcp*[r]{$\bigo(np)$}
        }
        \Return{
            $\beta^{ n_{\mathrm{iter}}}, h$
            }
        \label{alg:backward_pgd}
        \end{algorithm}
        }
    \end{minipage}
\end{figure}
%


\section{Bilevel Optimization with Non-Smooth Inner Problems}\label{sec:non_smooth}
%

We recalled above how to compute hypergradients when the inner optimization problem is smooth.
In this section we tackle the bilevel optimization \Cref{pb:bilevel_opt} with non-smooth inner optimization \Cref{pb:generic_inner_pb}.
Handling non-smooth inner problems requires specific tools detailed in \Cref{sub:framework}.
We then show how to compute gradients with non-smooth inner problems using implicit differentiation (\Cref{sub:implicit_nonsmooth}) or iterative differentiation (\Cref{sub:iterdiff_nonsmooth}).
In \Cref{sub:approximate_gradient} we tackle the problem of approximate gradient for a non-smooth inner optimization problem.
Finally, we propose in \Cref{sub:bilevel_non-smooth} an algorithm to solve the bilevel optimization \Cref{pb:bilevel_opt}.
%
\subsection{Theoretical Framework}
\label{sub:framework}
%
%
\textit{Differentiability of the regularization path.}
    Before applying first-order methods to tackle \Cref{pb:bilevel_opt}, one must ensure that the regularization path $\lambda \mapsto \hat \beta^{(\lambda)}$ is almost everywhere differentiable (as in \Cref{fig:intro_reg_path}). This is the case for
    the Lasso \citep{Mairal_Yu12} and the SVM \citep{Pontil_Verri98} since solution paths are piecewise differentiable (see \Cref{fig:intro_reg_path}).
    Results for nonquadratic datafitting terms are scarcer: \citet{Friedman_Hastie_Tibshirani10} address the practical resolution of sparse logistic regression, but stay evasive regarding the differentiability of the regularization path.
    In the general case for problems of the form \Cref{pb:generic_inner_pb}, we believe it is an open question and leave it for future work.
\begin{figure}[tb]
    \centering
    \begin{subfigure}[b]{1\textwidth}
        \includegraphics[width=1\linewidth]{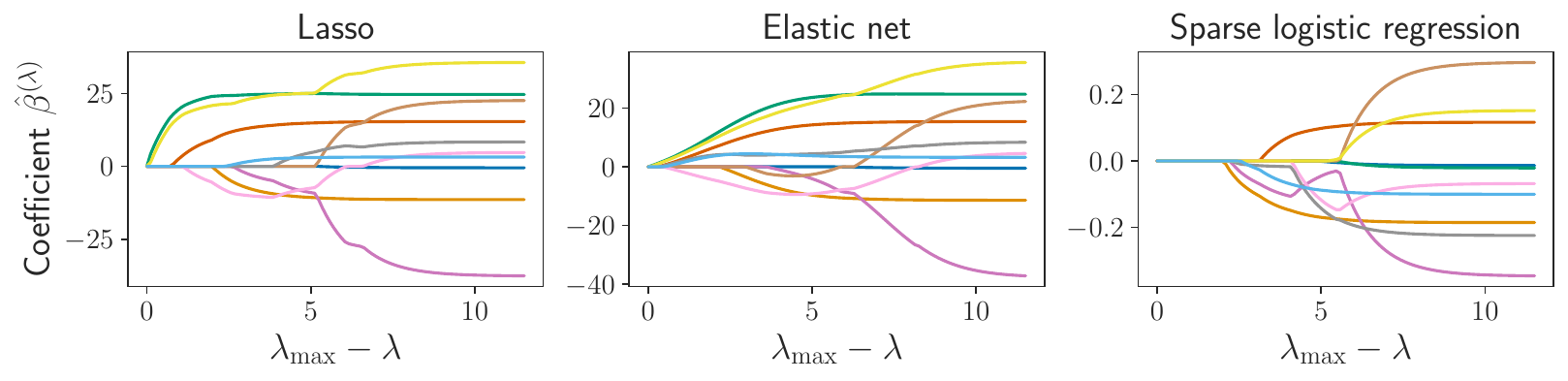}
    \end{subfigure}
    \caption{
        \emph{Regularization paths} (coefficient values as a function of $\lambda$),
        on the \emph{diabetes} and \emph{breast cancer} data sets
        for the Lasso, the elastic net and sparse logistic regression.
        This illustrates the weak differentiability of the paths.
        We used \emph{diabetes} for the Lasso and the elastic net, and the $10$ first features of \emph{breast cancer} for the sparse logistic regression.}
    \label{fig:intro_reg_path}
  \end{figure}
\emph{Differentiability of proximal operators.}
The key point to obtain an implicit differentiation formula for non-smooth inner
problems is to differentiate the fixed point equation of proximal gradient
descent.
From a theoretical point of view, ensuring this differentiability at the optimum is non-trivial: \citet[Thm. 3.8]{Poliquin_Rockafellar1996a} showed that under a \emph{twice epi-differentiability} condition the proximal operator is differentiable at optimum.
For the convergence of \forwardandbackward in the non-smooth case, one has to ensure that, after enough iterations, the updates of the algorithms become differentiable.
\citet{Deledalle_Vaiter_Fadili_Peyre14} justified (weak) differentiability of proximal operators as they are non-expansive.
However this may not be a sufficient condition, see \citet{Bolte_Pauwels2020a,Bolte_Pauwels2020b}.
In our case, we show differentiability after \emph{support identification} of the algorithms: active constraints are identified after a finite number of iterations by proximal gradient descent \citep{Liang_Fadili_Peyere14,Vaiter_Peyre_Fadili2018} and proximal coordinate descent, see \citet[Sec. 6.2]{Nutini2018} or \citet{Klopfenstein_Bertrand_Gramfort_Salmon_Vaiter20}.
Once these constraints have been identified convergence is linear towards the Jacobian (see \Cref{thm:iterdiff_linear_convergence,fig:linear_convergence_lasso,fig:linear_convergence_logreg,fig:linear_convergence_svm}).

For the rest of this paper, we consider the bilevel optimization \Cref{pb:bilevel_opt} with the following assumptions on the inner \Cref{pb:generic_inner_pb}.
\begin{assumption}
    \emph{Smoothness.}
    \label{ass:smoothness}
    The function $f:\mathbb{R}^p\rightarrow \mathbb{R}$ is a convex, differentiable function, with a $L$-Lipschitz gradient.
\end{assumption}
\begin{assumption}
    \emph{Proper, closed, convex.}
    \label{ass:proper}
    For all $\lambda \in \bbR^r$, for any $j\in [p]$, the function $g_j(\cdot, \lambda):\mathbb{R} \rightarrow \mathbb{R}$  is proper, closed and convex.
\end{assumption}
\begin{assumption}
    \emph{Non-degeneracy}
    \label{ass:non_degeneracy}
    The problem admits at least one solution
    \begin{equation*}
        \argmin_{\beta \in \bbR^p} \Phi(\beta, \lambda) \neq \emptyset
        \enspace ,
    \end{equation*}
    and, for any
    $\hat \beta $
    solution of \Cref{pb:generic_inner_pb},
    we have
    \begin{equation*}
        - \nabla f (\hat \beta)
        \in
        \ri\left(\partial_{\beta} g(\hat \beta, \lambda)\right)
        \enspace .
    \end{equation*}
\end{assumption}
To be able to extend iterative and implicit differentiation to the non-smooth case, we need to introduce the notion of generalized support.
\begin{definition}
    \emph{Generalized support, \citep[Def. 1]{Nutini_Schmidt_Hare2019}.}
    \label{def:gsupp}
    \sloppy
    For a solution $\hat{\beta}\in \argmin_{\beta \in \mathbb{R}^p} \Phi(\beta, \lambda)$, its \emph{generalized support} $\hat{S} \subseteq [p]$ is the set of indices $j \in [p]$ such that $g_j$ is differentiable at $\hat{\beta}_j$
    \begin{align*}
        \hat{S}
        \eqdef \{j \in [p]: \partial_{\beta} g_j(\hat{\beta}_j, \lambda) \text{ is a singleton}\}
        \enspace .
    \end{align*}
    An iterative algorithm is said to achieve \textit{finite support identification} if its iterates $\beta^{(k)}$ converge to $\hat \beta$, and there exists $K \geq 0$ such that for all $j \notin \hat S$, for all $k \geq K, \beta_j^{(k)} = \hat \beta_j$.
\end{definition}
\textit{Examples.}
For the $\ell_1$ norm (promoting sparsity), $g_j(\hat{\beta}_j, \lambda) = e^\lambda |\hat{\beta}_j|$,
the generalized support is $\hat{S} \eqdef \{ j \in [p]: \hat{\beta}_j \neq 0\}$.
This set corresponds to the indices of the non-zero coefficients, which is the usual support definition.
For the SVM estimator, $g_j(\hat{\beta}_j, \lambda) = \iota_{[0, e^\lambda]}(\hat{\beta_j})$.
This function is non-differentiable at $0$ and at $e^\lambda$.
The generalized support for the SVM estimator then corresponds to the set of indices such that $\hat{\beta}_j \in ] 0, e^\lambda [$.
\medskip

Finally, to prove local linear convergence of the Jacobian we assume regularity and strong convexity on the generalized support.
\begin{assumption}
    \emph{Locally $\mathcal{C}^2$ and $\mathcal{C}^3$.}
    \label{ass:taylor_expansion}
    The map $\beta \mapsto f(\beta)$ is locally $\mathcal{C}^3$ around $\hat{\beta}$.
    For all $\lambda \in \bbR^r$, for all $j \in \hat S$ the map $g_j(\cdot, \lambda)$ is locally $\mathcal{C}^2$ around $\hat{\beta}_j$.
\end{assumption}
\begin{assumption}
    \emph{Restricted injectivity.}
    \label{ass:restricted_injectivity}
    Let $\hat \beta$ be a solution of \Cref{pb:generic_inner_pb} and $\hat S$ its generalized support.
    The solution $\hat \beta$ satisfies the following restricted injectivity condition
    \begin{align*}
        \nabla^2_{\hat S, \hat S} f(\hat{\beta}) \succ 0 \enspace .
    \end{align*}
\end{assumption}
\Cref{ass:smoothness,ass:proper} are classical to ensure inner problems
can be solved using proximal algorithms.
\Cref{ass:non_degeneracy} can be seen as a generalization of constraint qualifications \citep[Sec.~1]{Hare_Lewis2007} and is crucial to ensure \emph{support identification}.
Note that \Cref{ass:non_degeneracy} is hard to verify in advance in practice, and hard to relax theoretically \citep{Fadili_Malick_Peyre2018}.
\Cref{ass:taylor_expansion,ass:restricted_injectivity} are classical for the analysis \citep{Liang_Fadili_Peyre17} and sufficient to derive rates of convergence for the Jacobian of the inner problem once the generalized support has been identified.
\Cref{ass:taylor_expansion} is met for usual quadratic and logistic losses, as well as for usual penalties ($\ell_1$, $\ell_1 + \ell_2$-squared, box constraints).
For instance for the Lasso \Cref{ass:restricted_injectivity} boils down to $X_{S}^\top X_{S} \succ 0$, which holds with probability one if the entries of $X$ are drawn from a continuous distribution \citep{Tibshirani13}.

The next lemma guarantees uniqueness of \Cref{pb:generic_inner_pb} under \Cref{ass:non_degeneracy,ass:restricted_injectivity}.
%
\begin{lemma}
    \emph{\citep[Prop. 4.1]{Liang_Fadili_Peyre17}.}
    \label{lem:unicity}
    Assume that there exists a neighborhood $\Lambda$ of $\lambda$ such that
    \Cref{ass:non_degeneracy,ass:restricted_injectivity} are satisfied for every $\lambda \in \Lambda$.
    Then for every $\lambda \in \Lambda$, \Cref{pb:generic_inner_pb} has a unique solution, and the map $\lambda \mapsto \hat \beta^{(\lambda)}$
    is well-defined on $\Lambda$.
\end{lemma}

We first show how implicit and iterative differentiation can be used with a non-smooth inner problem.
\citet{Peyre_Fadili11} proposed to smooth the inner optimization problem,
\citet{Ochs_Ranftl_Brox_Pock15,Frecon_Salzo_Pontil2018} relied on the \forward combined with Bregman iterations to get differentiable steps.
For non-smooth optimization problems, implicit differentiation has been considered for (constrained) convex optimization problems \citep{Gould_Fernando_Cherian_Anderson_Cruz_Guo2016,Amos_Kolter17,Agrawal_Amos_Barratt_Boyd_Diamond_Kolter19}, Lasso-type problems \citep{Mairal_Bach_Ponce12,Bertrand_Klopfenstein_Blondel_Vaiter_Gramfort_Salmon2020}, total variation penalties \citep{Cherkaoui_Sulam_Moreau2020} and generalized to strongly monotone operators \citep{Winston_Kolter2020}.

\subsection{Hypergradient Computation: Implicit Differentiation}
\label{sub:implicit_nonsmooth}
%
The exact proof of \Cref{thm:smooth_implicitdiff} cannot be applied when $\beta \mapsto \Phi(\beta, \lambda)$ is non-smooth, as \Cref{eq:grad,eq:fixed_point_smooth} no longer hold.
Nevertheless, instead of the optimality condition of smooth optimization,
\Cref{eq:fixed_point_smooth}, one can leverage the fixed point iteration of
proximal gradient descent, which we will see in \Cref{eq:fixed_point_pgd}.
The main theoretical challenge is to show the differentiability of the function
$\beta \mapsto \prox_{\gamma g} (\beta - \gamma \nabla f(\beta))$.
Besides, taking advantage of the generalized sparsity of the regression coefficients $\hat \beta^{(\lambda)}$, one can show that the Jacobian $\hat \jac$ is row-sparse, leading to substantial computational benefits when computing the hypergradient $\nabla_\lambda \cL(\lambda)$) for \Cref{pb:generic_inner_pb},
%
%
%
%
%
%
\begin{theorem}
    \emph{Non-smooth implicit formula}
    \label{thm:non-smooth_implicitdiff}
    \sloppy
    Suppose \Cref{ass:smoothness,ass:proper,ass:taylor_expansion} hold.
    Let $0 < \gamma \leq 1 / L$, where $L$ is the Lipschitz constant of $\nabla f$.
    Let $\lambda \in \bbR^r$, $\Lambda$ be a neighborhood of $\lambda$, and $\Gamma^{\Lambda} \eqdef  \condset{\hat \beta^{(\lambda)} - \gamma \nabla f (\hat \beta^{(\lambda)}) }{ \lambda \in \Lambda }$.
    In addition,
    \begin{enumerate}[label=(H\arabic*)]\setlength\itemsep{1pt}
        \item  \label{hyp:uniqueness} Suppose \Cref{ass:non_degeneracy,ass:restricted_injectivity} hold on $\Lambda$.
        \item  \label{hyp:diff_path} Suppose $\lambda \mapsto \hat \beta^{(\lambda)}$ is continuously differentiable on $\Lambda$.
        \item \label{hyp:prox_c1} Suppose for all
        $z \in \Gamma^{\Lambda}$,
        $\lambda \mapsto  \prox_{\gamma g(\cdot, \lambda)}(z) $
        is continuously differentiable on $\Lambda$.
        \item \label{hyp:prox_lip} Suppose $\partial_z \prox_{\gamma g(\cdot, \lambda)}$ and $\partial_{\lambda} \prox_{\gamma g(\cdot, \lambda)}$ are Lipschitz continuous on $ \Gamma^{\Lambda} \times \Lambda$.
    \end{enumerate}
    Let $\hat \beta \eqdef \hat \beta^{(\lambda)}$ be the solution of \Cref{pb:generic_inner_pb}, $\hat S$ its generalized support of cardinality $\hat{s}$.
    Then the Jacobian $\hat \jac$ of the inner \Cref{pb:generic_inner_pb} is given by the following formula,

    $\hat{z}
        =
        \hat \beta -
         \gamma \nabla f(\hat \beta )$,
    and
    $A \eqdef \Id_{\hat s}
    -
    \partial_z \prox_{\gamma g(\cdot, \lambda)}(\hat z)_{\hat S} \odot
    \left ( \Id_{\hat s} - \gamma \nabla^2_{\hat S,\hat S} f(\hat{\beta})\right )$
    \begin{align}
            \hat{\jac}_{\hat{S}^c:}
            &=
            \partial_{\lambda}\prox_{\gamma g(\cdot, \lambda)}\left(\hat{z}\right)_{\hat{S}^c} \enspace ,
            \label{eq:linear_system_nonsmooth_complement}
            \\
            \hat \jac_{\hat S:}
            &= A^{-1}
            \left (
                \partial_{\lambda} \prox_{\gamma g(\cdot, \lambda)}(\hat z)_{\hat S}
            - \gamma \partial_z \prox_{\gamma g(\cdot, \lambda)}(\hat z)_{\hat S}
            \odot \nabla^2_{\hat S,\hat S^c} f(\hat{\beta})
            \hat \jac_{\hat S^c}
            \right )
            \enspace .
            \label{eq:linear_system_nonsmooth}
    \end{align}
\end{theorem}

\begin{proof}
    According to \Cref{lem:unicity}, \Cref{ass:restricted_injectivity,ass:non_degeneracy} ensure \Cref{pb:generic_inner_pb} has a unique minimizer and $\lambda \mapsto \hat \beta^{(\lambda)}$ is well-defined on $\Lambda$.
    We consider the proximal gradient descent fixed point equation:
    \begin{align}\label{eq:fixed_point_pgd}
        \hat \beta^{(\lambda)}
        =
        \prox_{\gamma g_{(\cdot, \lambda)}}
        \left(
            \hat \beta^{(\lambda)}
            - \gamma  \nabla f(\hat \beta^{(\lambda)})
        \right)
        \enspace .
    \end{align}
    \sloppy
    Together with the conclusion of \Cref{lem:unicity},
    \Cref{ass:smoothness,ass:taylor_expansion},
    and given \ref{hyp:diff_path}, \ref{hyp:prox_c1} and \ref{hyp:prox_lip},
    we have that
    $\lambda
    \mapsto
    \psi \left (
        \beta^{(\lambda)} - \gamma \nabla f (\hat \beta^{(\lambda)}), \lambda \right)
        \eqdef \prox_{\gamma g(\cdot, \lambda)}
    \left(\hat \beta^{(\lambda)} - \gamma \nabla f(\hat \beta^{(\lambda)})\right)$
    is differentiable at $\lambda$.
    One can thus differentiate \Cref{eq:fixed_point_pgd} \wrt $\lambda$, which leads to
    \begin{align}\label{eq:update_jac_pgd}
        \hat{\jac}
        =
        \partial_z \prox_{\gamma g(\cdot, \lambda)}(\hat z)
        \odot
        \left( \Id
            - \gamma \nabla^2 f(\hat{\beta}) \right) \hat{\jac}
            + \partial_{\lambda} \prox_{\gamma g(\cdot, \lambda)}
            \left(
                \hat z
            \right)
        \enspace ,
    \end{align}
    with $ \hat z = \hat{\beta} - \gamma  \nabla f(\hat{\beta})$.
    In addition to $0 < \gamma < 1 / L \leq 1 / L_j$, the separability of $g$ and \Cref{ass:smoothness,ass:proper,ass:taylor_expansion,ass:non_degeneracy} ensure (see \Cref{{lemma:diff_prox}}) that for any $j \in \hat{S}^c$,
    \begin{align}\label{eq:constant_prox_supp_c}
        \partial_z \prox_{\gamma g_j(\cdot, \lambda)} \left(\hat{\beta}_j - \gamma \nabla_j f(\hat{\beta})\right) = 0 \enspace .
    \end{align}
    Plugging \Cref{eq:constant_prox_supp_c} into \Cref{eq:update_jac_pgd} ensures \Cref{eq:linear_system_nonsmooth_complement} for all $j \in \hat{S}^c$
    \begin{align} \label{eq:value_jac_supp_c}
        \hat{\jac}_{j:}
        =
        \partial_{\lambda} \prox_{\gamma g_j(\cdot, \lambda)}
        \left(
            \hat{\beta}_j - \gamma \nabla_j f(\hat{\beta})
        \right) \enspace .
    \end{align}
    Plugging \Cref{eq:constant_prox_supp_c,eq:value_jac_supp_c} into \Cref{eq:update_jac_pgd} shows that the Jacobian restricted on the generalized support $\hat{S}$ satisfies the following linear system
    \begin{align}
        \left(
            \text{Id}_{\hat s} -
            \partial_z \prox_{\gamma g(\cdot, \lambda)}
            \left(\hat z\right)_{\hat{S}}
            \odot \big(\text{Id}_{\hat s} - \gamma  \nabla^2_{\hat S, \hat S} f(\hat{\beta})
            \big)
        \right)
        &\hat{\jac}_{\hat{S}:}
        =
        \nonumber
        \\
        - \gamma \partial_z \prox_{\gamma g(\cdot, \lambda)}(\hat z)_{\hat{S}} \odot
        &\nabla^2_{\hat S, \hat S^c} f(\hat{\beta})
        \hat \jac_{\hat S^c :}
        +
        \partial_{\lambda} \prox_{g}(\hat z)_{\hat S :}
        \enspace .
        \nonumber
    \end{align}
    Since $0 < \gamma \leq 1 / L$,
    \begin{align}
        \normin{\partial_z \prox_{\gamma g(\cdot, \lambda)}(\hat z)_{\hat S} \odot
         ( \Id_{\hat s} - \gamma  \nabla^2_{\hat S, \hat S} f( \hat{\beta}) )}_2
        \nonumber
        &\leq
        \normin{\partial_z \prox_{\gamma g(\cdot, \lambda)}(\hat z)_{\hat S}}
        \cdot
        \normin{
         \Id_{\hat s} - \gamma  \nabla^2_{\hat S, \hat S} f( \hat{\beta}) }_2
        \\
        & < 1
        \enspace .
        \label{eq:norm_smaller_one}
    \end{align}
    Since \Cref{eq:norm_smaller_one} holds,
    $A \eqdef \Id_{\hat s}
    -
    \partial_z \prox_{\gamma g(\cdot, \lambda)}(\hat z)_{\hat S} \odot
    ( \Id_{\hat s} - \gamma  \nabla^2_{\hat S, \hat S} f( \hat{\beta}) )$
    is invertible, which leads to \Cref{eq:linear_system_nonsmooth}.
\end{proof}
\begin{remark}
    In the smooth case a $p \times p$ linear system is needed to compute the Jacobian in \Cref{eq:jac}.
    For non-smooth problems this is reduced to an $\hat s \times \hat s$ linear system ($\hat s \leq p$ being the size of the generalized support, \eg the number of non-zero coefficients for the Lasso).
    This leads to significant speedups in practice, especially for very sparse vector $\hat \beta^{(\lambda)}$.
\end{remark}
\begin{remark}
    To obtain \Cref{thm:non-smooth_implicitdiff} we differentiated the fixed point equation of proximal gradient descent, though one could differentiate other fixed point equations (such as the one from proximal coordinate descent).
    The value of the Jacobian $\hat \jac$ obtained with different fixed point equations would be the same, yet the associated systems could have different numerical stability properties.
    We leave this analysis to future work.
\end{remark}
%
%
\subsection{Hypergradient Computation: Iterative Differentiation}
\label{sub:iterdiff_nonsmooth}
Instead of implicit differentiation, it is also possible to use iterative differentiation on proximal solvers.
In section \Cref{sub:iterdiff_smooth} we presented \forwardandbackward differentiation of proximal gradient descent (\Cref{alg:forward_pgd,alg:backward_pgd}).
In this section we study the iterative differentiation of proximal coordinate descent (\Cref{alg:forward_pcd,alg:backward_pcd}).
To instantiate algorithms easily on problems such as the Lasso, partial derivatives of usual proximal operators can be found in \Cref{table:partial_derivatives_prox}.

For coordinate descent, the computation of the iterative Jacobian in a forward way involves differentiating the following update
\begin{align*}
    z_j
    &
    \leftarrow \beta_j - \gamma_j \nabla_j f(\beta)\\
    \beta_j
    &
    \leftarrow
    \prox_{\gamma_j g_j}
    \left(
        \beta_j - \gamma_j \nabla_j f(\beta)
        \right)\\
%
%
    \jac_{j:}
    &
    \leftarrow
    \underbrace{\partial_z \prox_{\gamma_j g_j(\cdot, \lambda)}(z_j)}_{\in \bbR}
        \underbrace{\left (
            \jac_{j:} - \gamma_j \nabla_{j :}^2 f(\beta)\jac
        \right)}_{\in \bbR^p}
        +
        \underbrace{\partial_{\lambda} \prox_{\gamma_j g_j(\cdot, \lambda)}(z_j)}_{\in \bbR^p} \enspace .
\end{align*}
%
We address now the convergence of the iterative Jacobian scheme, a question which remained open in \citet[Section 4.1]{Deledalle_Vaiter_Fadili_Peyre14}.
We show next that the \forward converges to the Jacobian in the non-smooth separable setting of this paper. Moreover, we prove that the iterative Jacobian convergence is locally linear after support identification.
%
\begin{figure}[tb]
    \begin{minipage}[t]{0.485\linewidth}
    {\fontsize{5}{4}\selectfont
    \begin{algorithm}[H]
    \SetKwInOut{Input}{input}
    \SetKwInOut{Init}{init}
    \SetKwInOut{Parameter}{param}
    \caption{\textsc{\Forward PCD }}
    \Input{
        $X \in \bbR^{n \times p},
        y \in \bbR^{n},
        \lambda \in \bbR^r,
        n_{\mathrm{iter}} \in \bbN$, $\beta \in \bbR^p$, $\jac \in \bbR^{p \times r}, \gamma_1, \dots, \gamma_p$
        }
    \tcp{jointly compute coef. \& Jacobian }

        \For{$k = 1,\dots, n_{\mathrm{iter}}$}{
            \For{$j = 1, \hdots, p$}{

                \tcp*[h]{update the regression coefficients}\\
                $z_j \leftarrow
                \beta_j - \gamma_j  \nabla_j f( \beta)$
                \tcp*[l]{CD step}

                    $\diff z_j \leftarrow
                    \jac_{j:} -
                    \gamma_j \nabla^2_{j:} f(\beta) \jac$

                $\beta_j \leftarrow \prox_{\gamma_j g_j(\cdot, \lambda)}(z_j)$
                \tcp*[l]{prox. step}

                \tcp*[h]{update the Jacobian}\\
                \tcp*[h]{diff. \wrt $\lambda$}\\
                $\jac_{j:}
                \leftarrow
                \partial_z \prox_{\gamma_j g_j(\cdot, \lambda)}(z_j) \diff z_j$

                $\jac_{j:} \pluseq
                \partial_{\lambda} \prox_{\gamma_j g_j(\cdot, \lambda)}(z_j)$
            }
        $\beta^{(k)} = \beta$

        $\jac^{(k)} = \jac$
        }
    $v = \nabla C (\beta)$

    \Return{$\beta^{ n_{\mathrm{iter}}}, \jac^\top v$}
    \label{alg:forward_pcd}
    \end{algorithm}
    }
    \end{minipage}\hfill
    \begin{minipage}[t]{0.49\linewidth}
        {\fontsize{5}{4}\selectfont
        \begin{algorithm}[H]
        \SetKwInOut{Input}{input}
        \SetKwInOut{Init}{init}
        \SetKwInOut{Parameter}{param}
        \caption{\textsc{\Backward PCD }}
        \Input{
            $X \in \bbR^{n \times p},
            y \in \bbR^{n},
            \lambda \in \bbR^r,
            n_{\mathrm{iter}} \in \bbN$, $\beta \in \bbR^p,
            \gamma_1, \dots, \gamma_p$
            }
    \tcp{compute coef.}
    \For{$k = 1,\dots, n_{\mathrm{iter}}$}{
            \For{$j = 1, \hdots, p$}{

                \tcp*[h]{update the regression coefficients}\\
                $z_j \leftarrow
                \beta_j - \gamma_j \nabla_j f(\beta)$
                \tcp*[l]{CD step}

                $\beta_j \leftarrow \prox_{\gamma_j g_j(\cdot, \lambda)}(z_j)$
                \tcp*[l]{prox. step}

                $\beta^{(k, j)} = \beta; z_j^{(k)} = z_j$
                \hspace{-0.5em}\tcp*[l]{store iterates}
                }
            }

        \tcp{compute gradient $g$ in a backward way}

        $v = \nabla C (\beta^{ n_{\mathrm{iter}}})$,
        $h = 0_{\bbR^r}$

        \For{$k = n_{\mathrm{iter}}, n_{\mathrm{iter}} - 1, \dots, 1$}{
            \For{$j = p, \dots, 1$}{
                $h \minuseq \gamma_j v_j
                \partial_{\lambda} \prox_{\gamma_j g_j(\cdot, \lambda)}
                \big( z_j^{(k)} \big) $

                $v_j \timeseq \partial_z \prox_{\gamma_j g_j(\cdot, \lambda)}\big( z_j^{(k)} \big) $

                $v
                \minuseq \gamma_j v_j \nabla_{j:}^2 f(\beta^{(k, j)})$
                \tcp*[l]{$\bigo(np)$}
            }

        }
        \Return{
            $\beta^{ n_{\mathrm{iter}}}, h$
            }
        \label{alg:backward_pcd}
        \end{algorithm}
        }
    \end{minipage}
\end{figure}
\begin{restatable}{theorem}{iterdifflinconv}
    \emph{Local linear convergence of the Jacobian.}
    \label{thm:iterdiff_linear_convergence}
    Let $0 < \gamma \leq 1 / L$.
    Suppose \Cref{ass:smoothness,ass:proper,ass:taylor_expansion}
    hold.
    Let $\lambda \in \bbR^r$, $\Lambda$ be a neighborhood of $\lambda$, and $\Gamma^{\Lambda} \eqdef  \condset{\hat \beta^{(\lambda)} - \gamma \nabla f (\hat \beta^{(\lambda)}) }{ \lambda \in \Lambda }$.
    In addition, suppose hypotheses \ref{hyp:uniqueness} to \ref{hyp:prox_lip} from \Cref{thm:non-smooth_implicitdiff} are satisfied and the sequence $(\beta^{(k)})_{k \in \bbN}$ generated by \Cref{alg:forward_pgd} (respectively by \Cref{alg:forward_pcd}) converges toward $\hat \beta$.

    Then, the sequence of Jacobians $(\jac^{(k)})_{k\geq 0}$ generated by the \forward differentiation of proximal gradient descent (\Cref{alg:forward_pgd})  (respectively by  \forward differentiation of proximal coordinate descent, \Cref{alg:forward_pcd}) converges locally linearly towards $\hat \jac$.
\end{restatable}
Proof of \Cref{thm:iterdiff_linear_convergence} can be found in \Cref{app:proof:nonsmooth_iterdiff}.
\def \figsize {1}
%
%
\begin{figure*}[tb]
    \centering
    \begin{subfigure}[b]{1\textwidth}
        \centering
        \includegraphics[width=0.42\linewidth]{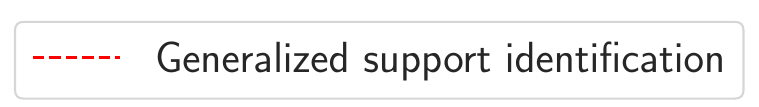}
        \includegraphics[width=\figsize\linewidth]{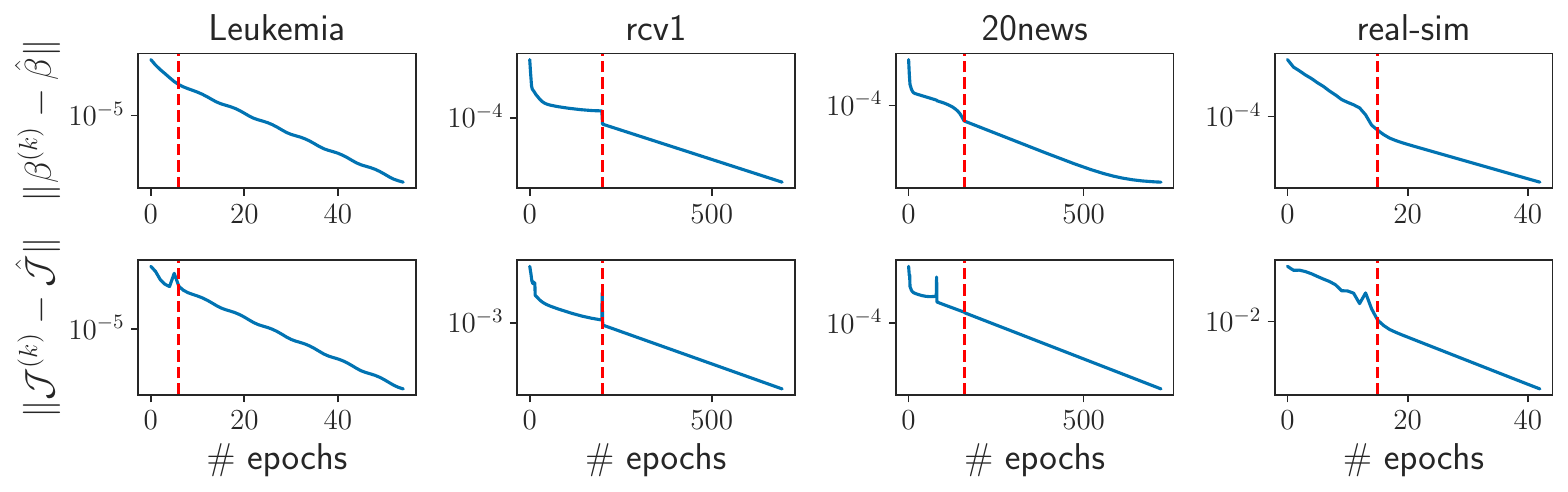}
    \end{subfigure}
    \caption{
        \emph{Local linear convergence of the Jacobian for the SVM.}
    Distance to optimum for the coefficients $\beta$ (top) and the Jacobian $\jac$ (bottom) of the \forward differentiation of proximal coordinate descent (\Cref{alg:forward_pcd}) on multiple data sets. One epoch corresponds to one
    pass over the data, \ie one iteration with proximal gradient descent.}
    \label{fig:linear_convergence_svm}
\end{figure*}


%
\textit{Comments on \Cref{fig:linear_convergence_svm}.}
We illustrate the results of \Cref{thm:iterdiff_linear_convergence} on SVM (for the Lasso and sparse logistic regression, see \Cref{fig:linear_convergence_lasso,fig:linear_convergence_logreg} in \Cref{app:sec:additional_xp}) for multiple data sets (\emph{leukemia}, \emph{rcv1}, \emph{news20} and \emph{real-sim}\footnote{Data available on the \emph{libsvm} website: \url{https://www.csie.ntu.edu.tw/~cjlin/libsvmtools/datasets/}}).
The values of the hyperparameters $\lambda$ are summarized in \Cref{table:setting_linear_cv}.
Regression coefficients $\hat \beta^{(\lambda)}$ were computed to machine
precision (up to duality gap smaller than $10^{-16}$) using  a state-of-the-art
coordinate descent solver implemented in \texttt{Lightning} \citep{Blondel_Pedregosa2016}.
The exact Jacobian was computed via implicit differentiation (\Cref{eq:linear_system_nonsmooth}).
Once these quantities were obtained, we used the \forward differentiation of proximal coordinate descent (\Cref{alg:forward_pcd}) and monitored the distance between the iterates of the regression coefficients $\beta^{(k)}$ and the exact solution $\hat \beta$.
We also monitored the distance between the iterates of the Jacobian $\jac^{(k)}$ and the exact Jacobian $\hat \jac$.
The red vertical dashed line represents the iteration number where support identification happens.
Once the support is identified, \Cref{fig:linear_convergence_lasso,fig:linear_convergence_logreg,fig:linear_convergence_svm} illustrate the linear convergence of the Jacobian.
However, the behavior of the iterative Jacobian before support identification is more erratic and not even monotone.

%
\subsection{Hypergradient Computation with Approximate Gradients}
\label{sub:approximate_gradient}
%
%
As mentioned in \Cref{sec:hypergrad_smooth}, relying on iterative algorithms to solve \Cref{pb:generic_inner_pb}, one only has access to an approximation of $\hat \beta^{(\lambda)}$:
this may lead to numerical errors when computing the gradient in \Cref{thm:non-smooth_implicitdiff}.
Extending the result of \citet[Thm. 1]{Pedregosa16}, which states that hypergradients can be computed approximately,
we give a stability result for the computation of approximate hypergradients in the case of non-smooth inner problems.
For this purpose we need to add several assumptions to the previous framework.
\begin{restatable}{theorem}{approxgrad}
    \emph{Bound on the error of approximate hypergradient.}
    \label{thm:approximate_gradient}
    For $\lambda \in \bbR^r$, let $\hat \beta^{(\lambda)} \in \bbR^p$ be the exact solution of the inner \Cref{pb:generic_inner_pb}, and $\hat S$ its generalized support.
    Suppose \Cref{ass:smoothness,ass:proper,ass:taylor_expansion}
    hold.
    Let $\Lambda$ be a neighborhood of $\lambda$, and $\Gamma^{\Lambda} \eqdef  \condset{\hat \beta^{(\lambda)} - \gamma \nabla f (\hat \beta^{(\lambda)}) }{ \lambda \in \Lambda }$.
    Suppose hypotheses \ref{hyp:uniqueness} to \ref{hyp:prox_lip} from \Cref{thm:non-smooth_implicitdiff} are satisfied.
    In addition suppose
    \begin{enumerate}[label=(H\arabic*)]\setlength\itemsep{1pt}
        \setcounter{enumi}{4}
        \item \label{hyp:lip_Hessian} The application $\beta \mapsto \nabla^{2} f(\beta)$ is Lipschitz continuous.
        \item \label{hyp:lip_criterion} The criterion
        $\beta \mapsto \nabla \mathcal{C}(\beta)$
        is Lipschitz continuous.
        \item \label{hyp:espilon_sol} Both optimization problems in \Cref{alg:implicit} are solved up to precision $\epsilon$ with support identification:
        $\normin{\beta^{(\lambda)} - \hat \beta^{(\lambda)} }
        \leq \epsilon$, $A^\top$ is invertible, and
        $ \normin{A^{-1 \top} \nabla_{\hat S} \cC(\beta^{(\lambda)}) - v} \leq \epsilon $.
    \end{enumerate}
    Then the error on the approximate hypergradient $h$ returned by \Cref{alg:implicit} is of the order of magnitude of the error $\epsilon$ on $\beta^{(\lambda)}$ and $v$
    \begin{equation*}
        \normin{\nabla \mathcal{L}(\lambda) - h}
        =
        \bigo (\epsilon)
        \enspace .
    \end{equation*}
\end{restatable}
Proof of \Cref{thm:approximate_gradient} can be found in \Cref{app:proof_approx_grad}.
Following the analysis of \citet{Pedregosa16}, two sources of approximation errors arise when computing the hypergradient:
one from the inexact computation of $\hat \beta$, and another from the approximate resolution of the linear system.
\Cref{thm:approximate_gradient} states that
if the inner optimization problem and the linear system are solved up to precision
$\epsilon$, \ie $\normin{\hat \beta^{(\lambda)} - \beta^{(\lambda)}} \leq \epsilon$
and
$\normin{A^{-1 \top } \nabla_S \cC(\beta^{(\lambda)}) - v} \leq \epsilon$,
then the approximation on the hypergradient is also of the order of $\epsilon$.

\begin{remark}
    The Lipschitz continuity of the proximity operator with respect to $\lambda$ \ref{hyp:prox_lip} is satisfied for usual proximal operators, in particular all the operators in \Cref{table:partial_derivatives_prox}.
   The Lipschitz continuity of the Hessian and the criterion, hypotheses \ref{hyp:lip_Hessian} and \ref{hyp:lip_criterion}, are satisfied for usual machine learning loss functions and criteria, such as the least squares and the logistic loss.
\end{remark}
\begin{remark}
    To simplify the analysis, we used the same tolerance for the resolution of the inner \Cref{pb:generic_inner_pb} and the resolution of the linear system.
    \Cref{thm:approximate_gradient} gives intuition on the fact that the inner problem does not need to be solved at high precision to lead to good hypergradients estimation.
    Note that in practice one does not easily control the distance between the approximate solution and the exact one $\normin{\beta^{(k)} - \hat \beta}$:
    most softwares provide a solution up to a given duality gap (sometimes even other criteria), not $\normin{\beta^{(k)} - \hat \beta}$.
\end{remark}
%
%
\subsection{Proposed Method for Hypergradient Computation}
\label{subsub:proposed_methods}

We now describe our proposed method to compute the hypergradient of \Cref{pb:bilevel_opt}.
In order to take advantage of the sparsity induced by the generalized support, we propose an implicit differentiation algorithm for non-smooth inner problem that can be found in \Cref{alg:implicit}.
First, we compute a solution of the inner \Cref{pb:generic_inner_pb} using a solver identifying the generalized support \citep{Liang_Fadili_Peyere14,Klopfenstein_Bertrand_Gramfort_Salmon_Vaiter20}.
Then, the hypergradient is computed by solving the linear system in \Cref{eq:linear_system_nonsmooth}.
This linear system, as mentioned in \Cref{sec:hypergrad_smooth}, can be solved using multiple algorithms, including conjugate gradient or fixed point methods.
\Cref{tab:summary_costs} summarizes the computational complexity in space and time of the described algorithms.

{\centering
\begin{table}[h]
  \caption{Cost in time and space for each method:
  $p$ is the number of features, $n$ the number of samples, $r$ the number of hyperparameters, and $\hat s$ is the size of the generalized support (\Cref{def:gsupp}, $\hat s \leq p$ and usually $\hat s \ll p$).
  The number of iterations of the inner solver is noted
  $\ninner$, the number of iterations of the solver of the linear system is noted  $n_{\text{sys}}$.}
  \label{tab:summary_costs}
  \centering
  \begin{tabular}{lc|cc}
    \toprule
     Differentiation
     & Algorithm
     & Space
     & Time\\
    \midrule
    \Forward PGD
    & \Cref{alg:forward_pgd}
    & $\bigo(p \, r)$
    & $\bigo(n \, p \, r \, \ninner)$
    \\
    \Backward PGD
    & \Cref{alg:backward_pgd}
    & $\bigo(p \, \ninner )$
    & $\bigo(n \, p \, \ninner + n \, p \, \ninner)$
    \\
    \Forward PCD
    & \Cref{alg:forward_pcd}
    & $\bigo(p \, r)$
    & $\bigo(n \, p \, r \, \ninner)$
    \\
    \Backward PCD
    & \Cref{alg:backward_pcd}
    & $\bigo(p \, \ninner )$
    & $\bigo(n \, p \, \ninner + n \, p^2 \, \ninner)$
    \\
    \Implicitfull
    & \Cref{alg:implicit}
    &  $\bigo(p + \hat s )$
    & $\bigo(n \, p \, \ninner + n \, \hat s \, \nlinsys)$
    \\
    \bottomrule
  \end{tabular}
\end{table}
}

\subsection{Resolution of the Bilevel Optimization \Cref{pb:bilevel_opt}}
\label{sub:bilevel_non-smooth}
%
From a practical point of view, once the hypergradient has been computed, first-order methods require the definition of a step size to solve the non-convex \Cref{pb:bilevel_opt}.
As the Lipschitz constant is not available for the outer problem, first-order methods need to rely on other strategies, such as:
\begin{itemize}
    \item Gradient descent with manually adjusted fixed step sizes \citep{Frecon_Salzo_Pontil2018,Ji_Yang_Liang2020}.
    The main disadvantage of this technique is that it requires a careful tuning of the step size for each experiment.
    In addition to being potentially tedious, it does not lead to an automatic procedure.
    \item L-BFGS (as in \citealt{Deledalle_Vaiter_Fadili_Peyre14}). L-BFGS is a quasi-Newton algorithm that exploits past iterates to approximate the Hessian and propose a better descent direction, which is combined with some line search~\citep{Nocedal_Wright06}.
    Yet, due to the approximate gradient computation, we observed that L-BFGS did not always converge.
    \item ADAM \citep{Kingma_Ba2014}. It turned out to be inappropriate to the present setting. ADAM was very sensitive to the initial step size and required a careful tuning for each experiment.
    \item Iteration specific step sizes obtained by line search \citep{Pedregosa16}. While the approach from \citet{Pedregosa16} requires no tuning, we observed that it could diverge when close to the optimum.
    The
    normalized gradient strategy \citep[Sec. 3.9]{Watt_Borhani_Katsaggelos2020}\footnote{\url{https://jermwatt.github.io/machine_learning_refined/notes/3_First_order_methods/3_9_Normalized.html}}
    proposed in \Cref{alg:bilevel_approx}, used in all the experiments, turned out to be robust and efficient across problems and data sets.
\end{itemize}
\begin{remark} \emph{Uniqueness.}
    The solution of \Cref{pb:generic_inner_pb} may be non-unique, leading to a multi-valued regularization path $\lambda \mapsto \hat \beta^{(\lambda)}$ \citep{Liu_Mu_Yuan_Zeng_Zhang2020} and requiring tools such as \emph{optimistic gradient} \citep[Chap. 3.8]{Dempe_Kalashnikov_Perez-Valdes_Kalashnykova15}.
    Though it is not possible to ensure uniqueness in practice, we did not face experimental issues due to potential non-uniqueness.
    For the Lasso, this experimental observation can be theoretically justified \citep{Tibshirani13}: when the design matrix is sampled from a continuous distribution, the solution of the Lasso is almost surely unique.
\end{remark}
\begin{remark}\emph{Initialization and warm start.}
    \sloppy
    One advantage of the non-smooth case with the $\ell_1$ norm is that one can find a good initialization point:
    there exists a value $\lambda_{\max}$ (see \Cref{table:inner}) such that the solution of \Cref{pb:generic_inner_pb} vanishes for $\lambda \geq \lambda_{\max}$.
    Hence, a convenient and robust initialization value can be chosen as $e^\lambda = e^{\lambda_{\max}} / 100$.
    This is in contrast with the smooth case, where finding a good initialization heuristic is hard: starting in flat zones can lead to poor performance for gradient-based methods \citep{Pedregosa16}.
    \Cref{alg:implicit} is called multiple times in \Cref{alg:bilevel_approx}: several inner optimization problems and linear systems which are "similar" are solved successively.
    That is why we use warm-start to solve these problems.
\end{remark}
%
%
\begin{figure}[tb]
    \begin{minipage}[t]{0.5\linewidth}

{\fontsize{5}{4}\selectfont
\begin{algorithm}[H]
\SetKwInOut{Input}{input}
\SetKwInOut{Init}{init}
\SetKwInOut{Parameter}{param}
\caption{\textsc{\Implicitfull}}
\Input{
$
\lambda \in \bbR,
\epsilon > 0
$}
\Init{$\gamma > 0$}

\tcp*[l]{compute the solution of inner problem}
Find $\beta$ such that:
$\Phi(\beta, \lambda) - \Phi(\hat \beta, \lambda) \leq \epsilon$

\tcp*[l]{compute the gradient}

Compute the generalized support $S$ of $\beta$,

$ z = {\beta} - \gamma \nabla f(\beta)$

$\jac_{S^c :}
=
\partial_{\lambda} \prox_{\gamma g(\cdot, \lambda)}( z)_{S^c}$

$s=|S|$

$A \hspace{-1mm} = \hspace{-1mm} \Id_s
-
\partial_z \prox_{\gamma g(\cdot, \lambda)}(z)_{S} \odot
( \Id_{s} - \gamma \nabla^2_{S, S} f(\beta)
  )$

Find $v\in\bbR^s$ s.t.
$\normin{A^{-1 \top } \nabla_S \cC(\beta) - v} \leq \epsilon$

$B =
\partial_{\lambda} \prox_{\gamma g(\cdot, \lambda)}( z)_{ S} \\
 \hspace*{4mm} - \gamma \partial_z \prox_{\gamma g(\cdot, \lambda)}( z)_{ S} \odot
\nabla^2_{S, S^c} f(\beta)
\jac_{ S^c}$

$\nabla \cL(\lambda)
= \jac_{S^c :}^\top \nabla_{S^c} \cC(\beta)
+ v^\top B
$

\Return{
    $\cL(\lambda) \eqdef \cC(\beta), \nabla \cL(\lambda)$
    }
\label{alg:implicit}
\end{algorithm}
}
    \end{minipage}\hfill
    \begin{minipage}[t]{0.49\linewidth}
        {\fontsize{5}{4}\selectfont
        \begin{algorithm}[H]
        \SetKwInOut{Input}{input}
        \SetKwInOut{Init}{init}
        \SetKwInOut{Parameter}{param}
        \caption{
            \textsc{Gradient descent with approximate gradient} }
        \Input{
            $
            \lambda \in \bbR^r,
            (\epsilon_i)$ }
        \Init{$\textrm{use\_adaptive\_step\_size} = \textrm{True}$}
            \For{$i = 1,\dots, \mathrm{iter}$}{
                $\lambda^{\mathrm{old}} \leftarrow \lambda $

                \tcp*[h]{compute the value and the gradient}
                $\cL(\lambda), \nabla \cL(\lambda) \leftarrow {\rm \Cref{alg:implicit}}(X, y, \lambda, \epsilon_i) $

                \If{$\mathrm{use\_adaptive\_step\_size}$}{
                    $\alpha = 1 / \normin{\nabla \cL(\lambda)} $
                }

                $\lambda \minuseq \alpha \nabla \cL(\lambda)$
                \tcp*[l]{gradient step}

                \If{$\cL(\lambda) > \cL( \lambda^{\mathrm{old}} ) $}{
                $\mathrm{use\_adaptive\_step\_size} = \mathrm{False}$

                $\alpha \diveq 10$
                    }
            }

        \Return{
            $\lambda$
            }
        \label{alg:bilevel_approx}
        \end{algorithm}
        }
    \end{minipage}
\end{figure}
    \begin{remark}\emph{Role of the step size $\gamma$ in \Cref{alg:implicit}.}
        In all the convex penalties we used ($\ell_1$-norm, $\ell_1 + \ell_2$-squared norm, indicator function) the step size $\gamma$ simplifies and does not appear in the \implicitfull formula.
        Instantiations of \Cref{alg:implicit} for the Lasso, the elastic net, the weighted Lasso and the dual of the SVM can be found in \Cref{app:algos_implicit_diff}.
    \end{remark}

%
\section{Experiments}
\label{sec:experiments}
%

In this section, we illustrate the benefits of our proposed
\Cref{alg:implicit} to compute hypergradients and \Cref{alg:bilevel_approx} to solve \Cref{pb:bilevel_opt}.
Our package, \texttt{sparse-ho}, is implemented in \texttt{Python}.
It relies on \texttt{Numpy} \citep{numpy}, \texttt{Numba} \citep{Lam_Pitrou_Seibert15} and \texttt{SciPy} \citep{scipy}.
Figures were plotted using \texttt{matplotlib} \citep{matplotlib}.
The package is available under BSD3 license at \url{https://github.com/qb3/sparse-ho}, with documentation and examples available at \url{https://qb3.github.io/sparse-ho/}.
Online code includes scripts to reproduce all figures and experiments of the paper.
\begin{table}[t!]
    \centering
    \caption{
        Characteristics of the data sets used for the experiments.}
    \begin{tabular}{ccccc}

    \hline
    name
    & $\# \text{ samples } n$
    & $\# \text{ features } p$
    & $\# \text{ classes } q$
    & density \\
    \hline
    \emph{breast cancer}
    & $569$
    & $\num{30}$
    & $-$
    & $1$ \\
    \emph{diabetes}
    & $442$
    & $\num{10}$
    & $-$
    & $1$ \\
    \emph{leukemia}
    & $72$
    & $\num{7129}$
    & $-$
    & $1$ \\
    \emph{gina agnostic}
    & $\num{3468}$
    & $\num{970}$
    & $-$
    & $1$ \\
    \emph{rcv1}
    & $\num{20242}$
    & $\num{19960}$
    & $-$
    & $3.7\times 10^{-3}$
    \\
    \emph{real-sim}
    & $\num{72309}$
    & $\num{20958}$
    & $-$
    & $2.4 \times 10^{-3}$ \\
    \emph{news20}
    & $\num{19996}$
    & $\num{632983}$
    & $-$
    & $6.1\times 10^{-4}$ \\
    \emph{mnist}
    & $\num{60,000}$
    & $\num{683}$
    & $10$
    & $2.2\times 10^{-1}$ \\
    \emph{usps}
    & $\num{7291}$
    & $256$
    & $10$
    & $1$ \\
    \emph{rcv1 (multiclass)}
    & $\num{15564}$
    & $\num{16245}$
    & $53$
    & $4.0 \times 10^{-3}$ \\
    \emph{aloi}
    & $\num{108000}$
    & $128$
    & $\num{1000}$
    & $2.4 \times 10^{-1}$ \\
    \hline
    \end{tabular}
\end{table}
\subsection{Hypergradient computation}
\label{sub:expe_hypergrad_computation}
%
%
\textit{Comparison with alternative approaches (\Cref{fig:hypergradient_cvxpy}).}
%
\begin{figure*}[tb]
    \centering
    \begin{subfigure}[b]{1\textwidth}
        \centering
        \includegraphics[width=0.75\linewidth]{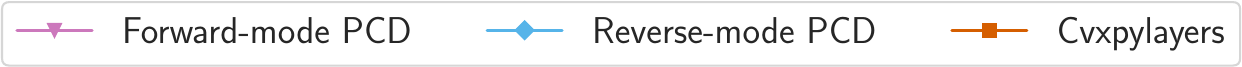}
        \includegraphics[width=\figsize\linewidth]{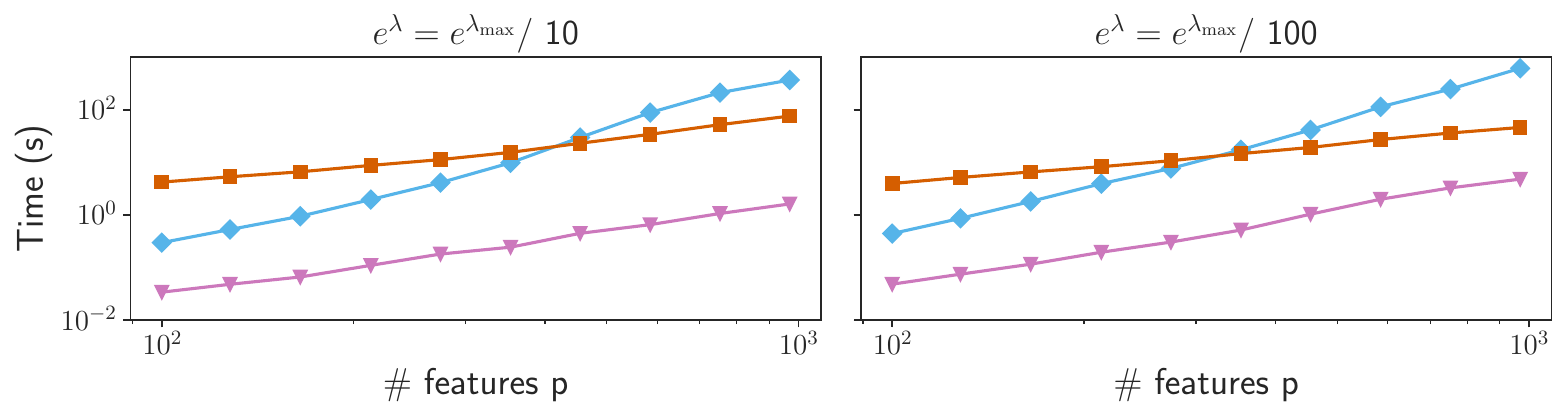}
    \end{subfigure}
    \caption{
        \emph{Lasso with hold-out criterion:}
        time comparison on the gina data set to compute a single hypergradient as
        a function of the number of features, for two values of $\lambda$,
        $e^\lambda = e^{\lambda_{\max}} / 10$ (left) and
        $e^{\lambda} = e^{\lambda_{\max}} / 100$ (right).
    }
    \label{fig:hypergradient_cvxpy}
\end{figure*}
First, we compare different methods to compute the hypergradient:
\begin{itemize}
    \setlength\itemsep{0pt}
    \item \Forward differentiation of proximal coordinate descent (\Cref{alg:forward_pcd}).
    \item \Backward differentiation of proximal coordinate descent (\Cref{alg:backward_pcd}).
    \item \texttt{cvxpylayers}
    \citep{Agrawal_Amos_Barratt_Boyd_Diamond_Kolter19}, a software based on \texttt{cvxpy}
    \citep{cvxpy}, solving \emph{disciplined parametrized programming} and providing derivatives with respect to the parameters of the program.
    It is thus possible to use \texttt{cvxpylayers} to compute gradients with respect to the regularization parameters.
\end{itemize}
\Cref{fig:hypergradient_cvxpy} compares the time taken by multiple methods to compute a single hypergradient $\nabla \cL(\lambda)$ for the Lasso (see \Cref{table:inner}), for multiple values of $\lambda$.
It shows the time taken to compute the regression coefficients and the hypergradient, as a function of the number of columns, sampled from the design matrix from the \emph{gina} data set.
The columns were selected at random and $10$ repetitions were performed for each point of the curves.
In order to aim for good numerical precision, problems were solved up to a duality gap of $10^{-6}$ for the \forward and the \backward.
\texttt{cvxpylayers} relies on \texttt{cvxpy}, solving  \Cref{pb:generic_inner_pb} using a splitting conic solver \citep{scs}.
Since the termination criterion of the splitting conic solver is not exactly the duality gap \citep[Sec. 3.5]{Donoghue_Chu_Parikh_Boyd16}, we used the default tolerance of $10^{-4}$.
The hypergradient $\nabla \cL(\lambda)$ was computed for hold-out mean squared error (see \Cref{table:criterion}).

The \forward differentiation of proximal coordinate descent is one order of magnitude faster than \texttt{cvxpylayers} and two orders of magnitude faster than the \backward differentiation of proximal coordinate descent.
The larger the value of $\lambda$, the sparser the coefficients $\beta$ are, leading to significant speedups in this regime.
This performance is in accordance with the lower time cost of the forward mode in \Cref{tab:summary_costs}.

\textit{Combining implicit differentiation with state-of-the art solvers
(\Cref{fig:hypergradient_computation,fig:hypergradient_computation_svm}).}
We now compare the different approaches described in \Cref{sec:non_smooth}:
\begin{itemize}
    \setlength\itemsep{0pt}
    \item \Forward differentiation of proximal coordinate descent (\Cref{alg:forward_pcd}).
    \item \Implicitfull (\Cref{alg:implicit}) with proximal coordinate descent to solve the inner problem.
    For efficiency, this solver was coded in \texttt{Numba} \citep{Lam_Pitrou_Seibert15}.
    \item \Implicitfull (\Cref{alg:implicit}) with state-of-the-art algorithm to solve the inner problem: we used \texttt{Celer} \citep{Massias_Vaiter_Gramfort_Salmon20} for the Lasso, and \texttt{Lightning} \citep{Blondel_Pedregosa2016} for the SVM.
\end{itemize}
\begin{figure*}[tb]
    \centering
    \begin{subfigure}[b]{1\textwidth}
        \centering
        \includegraphics[width=0.82\linewidth]{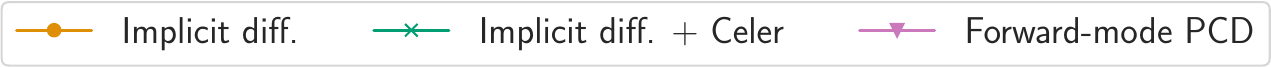}
        \includegraphics[width=\figsize\linewidth]{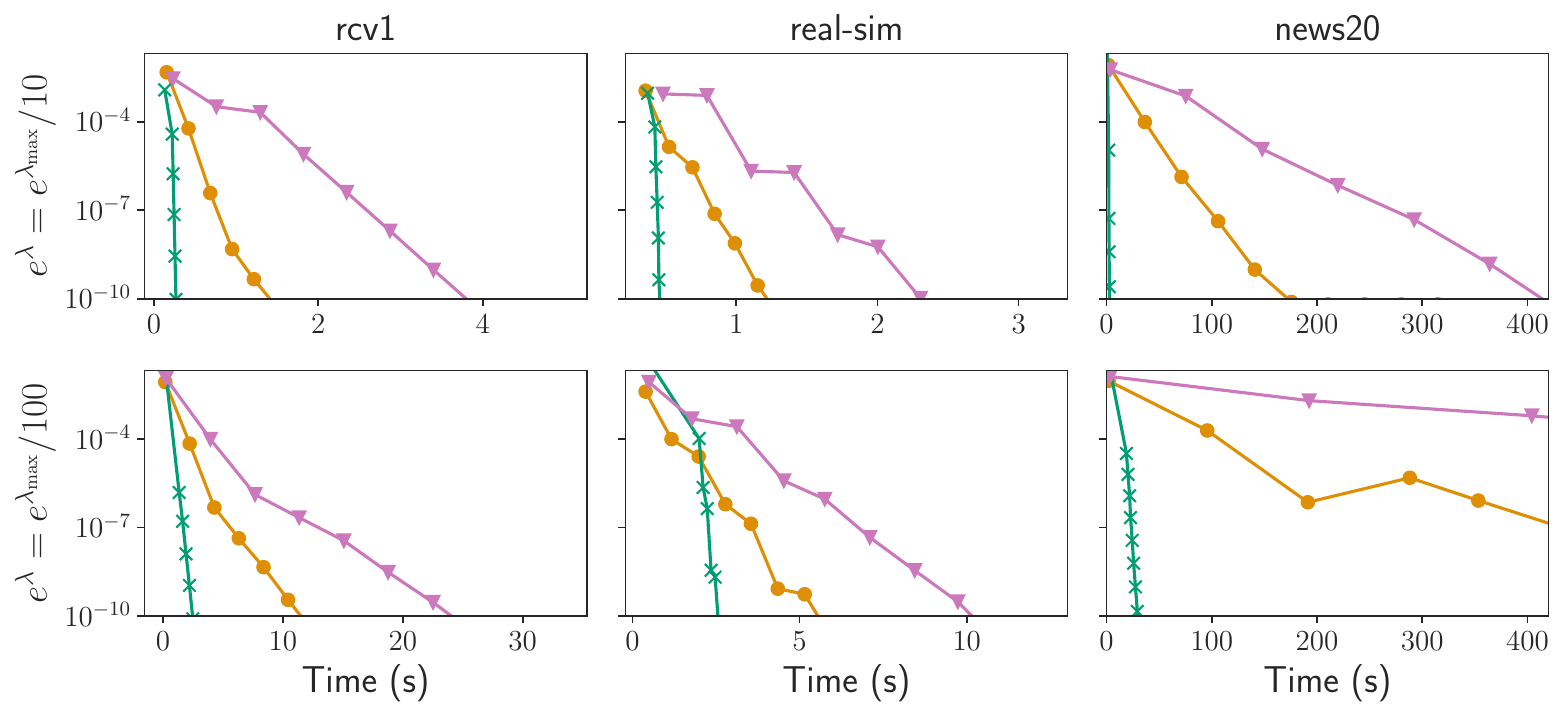}
    \end{subfigure}
    \caption{
        \emph{Lasso with hold-out criterion:}
        absolute difference between the exact hypergradient (using $\hat \beta$) and the iterate hypergradient (using $\beta^{(k)}$) of the Lasso as a function of time. Results are for three data sets and two different regularization parameters.
        ``Implicit diff. + \texttt{Celer})'' uses \texttt{Celer}
        \citep{Massias_Vaiter_Gramfort_Salmon20} instead of our proximal
    coordinate descent implementation.}
    \label{fig:hypergradient_computation}
\end{figure*}
\begin{figure*}[tb]
    \centering
    \begin{subfigure}[b]{1\textwidth}
        \centering
        \includegraphics[width=0.87\linewidth]{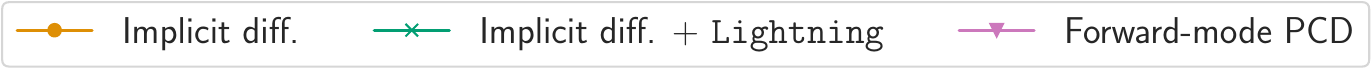}
        \includegraphics[width=\figsize\linewidth]{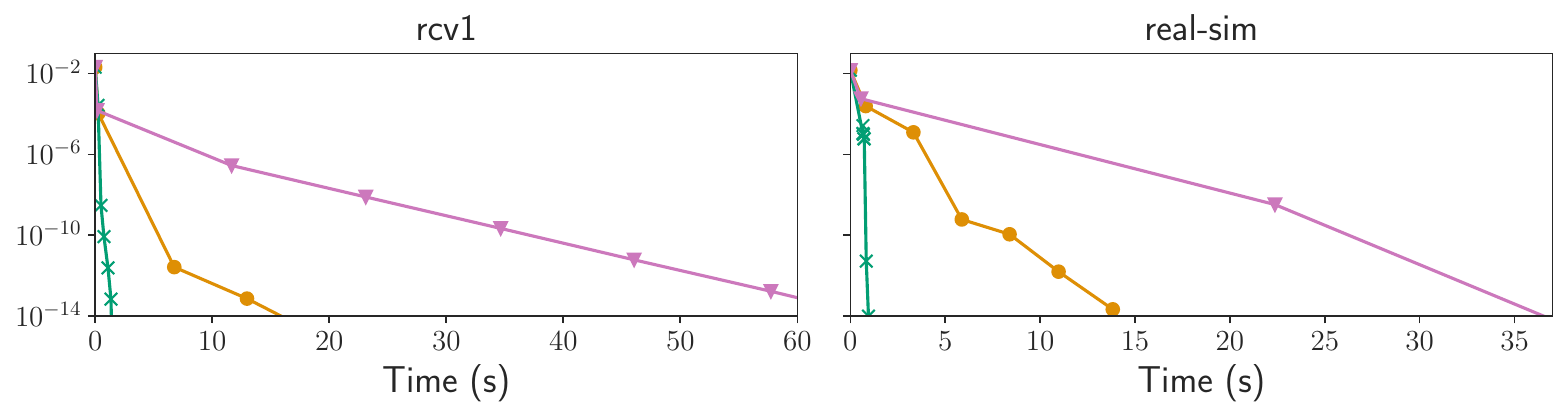}
    \end{subfigure}
    \caption{
        \emph{SVM with hold-out criterion:}
        absolute difference between the exact hypergradient (using $\hat \beta$) and the iterate hypergradient (using $\beta^{(k)}$) of the SVM as a function of time.
        ``Implicit diff. + \texttt{Lightning}'' uses \texttt{Lightning}
\citep{Blondel_Pedregosa2016}, instead of our proximal coordinate descent implementation.}
    \label{fig:hypergradient_computation_svm}
\end{figure*}
\Cref{fig:hypergradient_computation} shows for three data sets and two values of regularization parameters the absolute difference between the exact hypergradient and the approximate hypergradient obtained via multiple algorithms as a function of time.
\Cref{fig:hypergradient_computation_svm} reports similar results for the SVM, on the same data sets, except \emph{news20}, which is not well suited for SVM, due to limited number of samples.

First, it demonstrates that \implicitfull methods are faster than the \forward of proximal coordinate descent (pink).
This illustrates the benefits of restricting the gradient computation to the support of the Jacobian, as described in \Cref{subsub:proposed_methods}.
Second, thanks to the flexibility of our approach, we obtain additional speed-ups by combining \implicitfull with a state-of-the-art solver, \texttt{Celer}.
The resulting method (orange) significantly improves over \implicitfull using a vanilla proximal coordinate descent (green).
%
%
%
\subsection{Resolution of the Bilevel Optimization Problem}
\label{sub:expe_bilevel_optimization}
%
In this section we compare multiple methods to find the optimal hyperparameters for the Lasso, elastic net and multiclass sparse logistic regression.
The following methods are compared:
\begin{itemize}
    \setlength\itemsep{-1pt}
    \item \textit{Grid-search}: for the Lasso and the elastic net, the number of hyperparameters is small, and grid-search is tractable.
    For the Lasso we chose a grid of $100$ hyperparameters $\lambda$, uniformly spaced between $\lambda_{\max} - \ln(10^4)$ and $\lambda_{\max}$.
    For the elastic net we chose for each of the two hyperparameters a grid of 10 values uniformly spaced between $\lambda_{\max}$ and $\lambda_{\max} - \ln(10^4)$.
    The product grid thus has $10^2$ points.
    \item \textit{Random-search}: we chose $30$ values of $\lambda$ sampled uniformly between $\lambda_{\max}$ and $\lambda_{\max} - \ln(10^4)$ for each hyperparameter.
    For the elastic net we chose $30$ points sampled uniformly in
    $[\lambda_{\max} - \ln(10^4), \lambda_{\max}]\times [\lambda_{\max} - \ln(10^4), \lambda_{\max}]$.
    \item \textit{SMBO}:
    this algorithm is SMBO using
    as criterion expected improvement (EI) and the Tree-structured Parzen Estimator (TPE) as model.
    First it evaluates $\cL$ using $5$ values of $\lambda$, chosen uniformly at random between $\lambda_{\max}$ and $\lambda_{\max} - \ln(10^4)$.
    Then a TPE model is fitted on the data points $(\lambda^{(1)}, \cL(\lambda^{(1)})), \dots, (\lambda^{(5)}, \cL(\lambda^{(5)}) )$.
    Iteratively, the EI is used to choose the next point to evaluate $\cL$ at, and this value is used to update the model.
    We used the \texttt{hyperopt} implementation \citep{Bergstra13}.
    \item \textit{1st order}: first-order method with exact gradient (\Cref{alg:bilevel_approx} with constant tolerances $\epsilon_i=10^{-6}$), with $\lambda_{\max} - \ln(10^2)$ as a starting point.
    \item \textit{1st order approx}: a first-order method using approximate gradient (\Cref{alg:bilevel_approx} with tolerances $\epsilon_i$, geometrically decreasing from $10^{-2}$ to $10^{-6}$), with $\lambda_{\max} - \ln(10^2)$ as a starting point.
\end{itemize}
\textit{Outer criterion.} In the Lasso and elastic net experiments, we pick a $K$-fold CV loss as outer criterion\footnote{In our experiments the default choice is $K=5$.}.
Hence, the data set $(X, y)$ is partitioned into $K$ hold-out data sets
$(X^{\text{train}_k}, y^{\text{train}_k}),
(X^{\text{val}_k}, y^{\text{val}_k})$.
The bilevel optimization problems then write
\def \figsize {0.95}
\begin{figure}[!tb]
    \centering
    \begin{subfigure}[b]{1\textwidth}
        \centering
        \includegraphics[width=\figsize\linewidth]{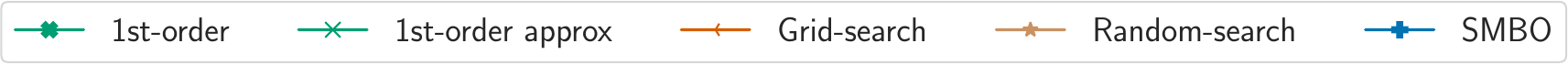}
        \includegraphics[width=\figsize\linewidth]{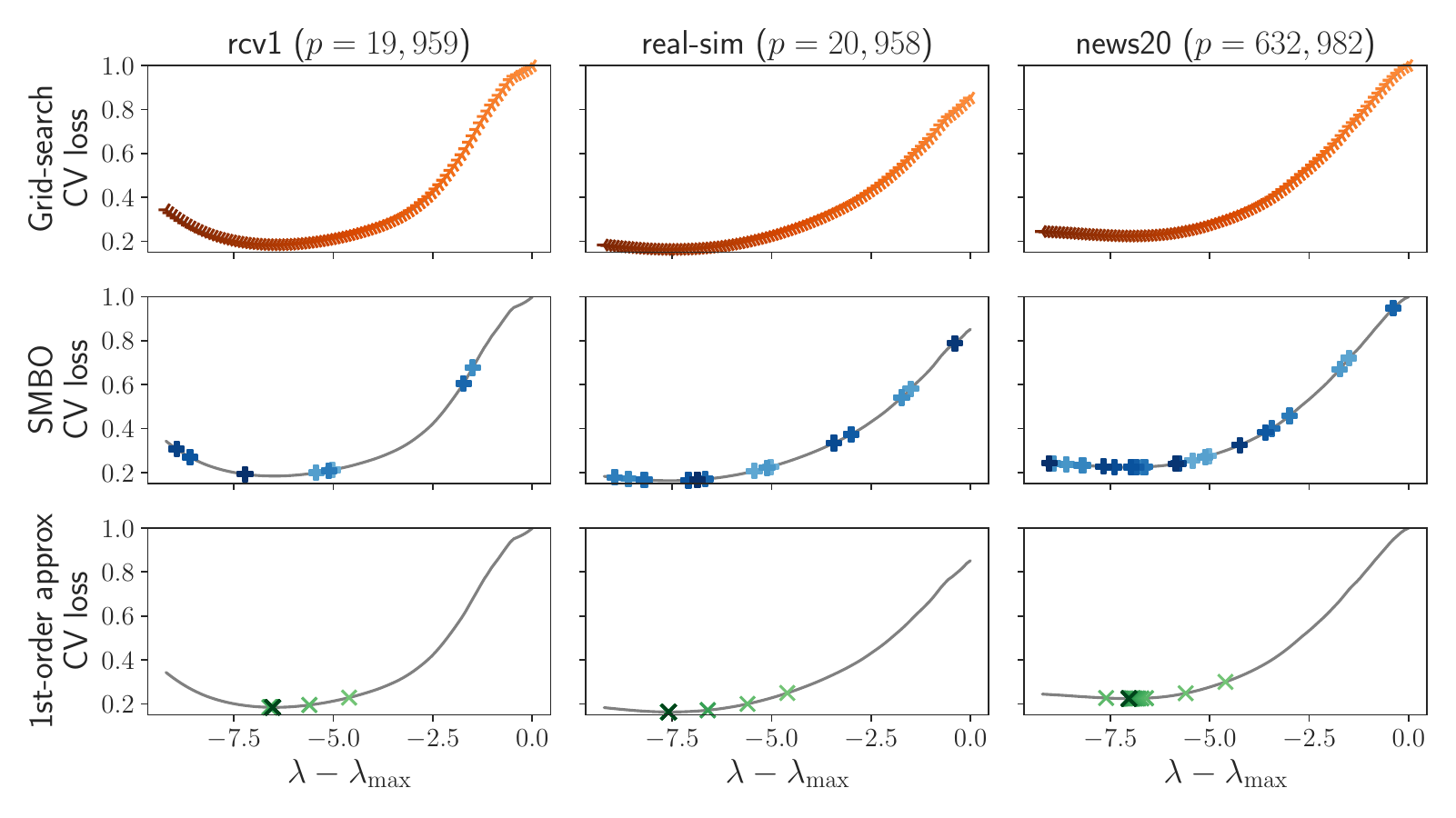}
        \includegraphics[width=\figsize\linewidth]{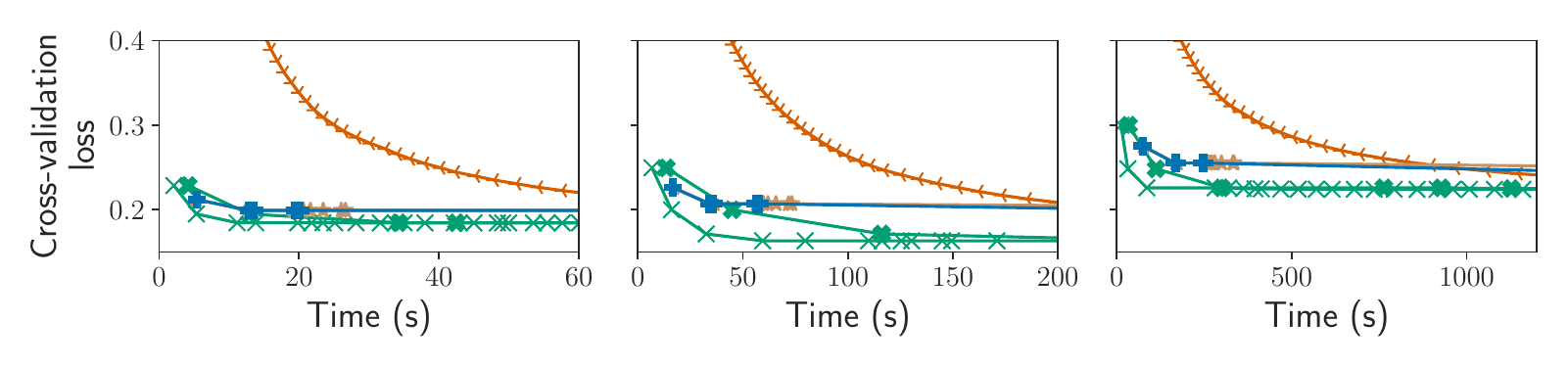}
     \end{subfigure}
    \caption{
        \emph{Lasso with cross-validation criterion:}
        cross-validation loss as a function of $\lambda$ (black line, top) and as a function of time (bottom).
        Lighter markers correspond to earlier iterations of the algorithm.}
    \label{fig:ho_lasso}
\end{figure}
\begin{equation}\label[pb_multiline]{pb:bilevel_enet_crosval}
    \begin{aligned}
    &
    \argmin_{\lambda = (\lambda_1, \lambda_2) \in \bbR^2}
    \cL(\lambda)
    =
    \frac{1}{K}
    \sum_{k=1}^K \normin{
        y^{\text{val}_k} - X^{\text{val}_k}
        \hat{\beta}^{(\lambda, k)}
        }^2_2
    \\
    &
    \st
    \hat \beta^{(\lambda, k)}
    \in
    \argmin_{\beta\in \bbR^p}
    \tfrac{1}{2n}
    \norm{y^{\text{train}_k}
    - X^{\text{train}_k}\beta}^2_2
    + e^{\lambda_1} \normin{\beta}_1
    + \frac{e^{\lambda_2}}{2} \normin{\beta}_2^2, \quad \forall k \in [K]
     \enspace,
    \end{aligned}
\end{equation}
while Lasso CV is obtained taking $\lambda_2\to -\infty$ in the former.
By considering an extended variable $\beta \in \bbR^{K \times p}$, cross-validation can be cast as an instance of \Cref{pb:bilevel_opt}.

\Cref{fig:ho_lasso} represents the cross-validation loss in Lasso CV
as a function of the regularization parameter $\lambda$ (black curve, three top rows) and as a function of time (bottom).
Each point corresponds to the evaluation of the cross-validation criterion for one $\lambda$ value.
The top rows show cross-validation loss as a function of $\lambda$, for the grid-search, the SMBO optimizer and the first-order method.
The lightest crosses correspond to the first iterations of the algorithm and the darkest, to the last ones.
For instance, Lasso grid-search starts to evaluate the cross-validation function with $\lambda=\lambda_{\max}$ and then decreases to $\lambda=\lambda_{\max} - \ln(10^4)$.
On all the data sets, first-order methods are faster to find the optimal regularization parameter, requiring only $5$ iterations.

\def \figsize {1}
\begin{figure}[tb]
    \centering
    \begin{subfigure}[b]{0.9\textwidth}
        \centering
        \includegraphics[width=\figsize\linewidth]{lasso_val_legend}
        \includegraphics[width=\figsize\linewidth]{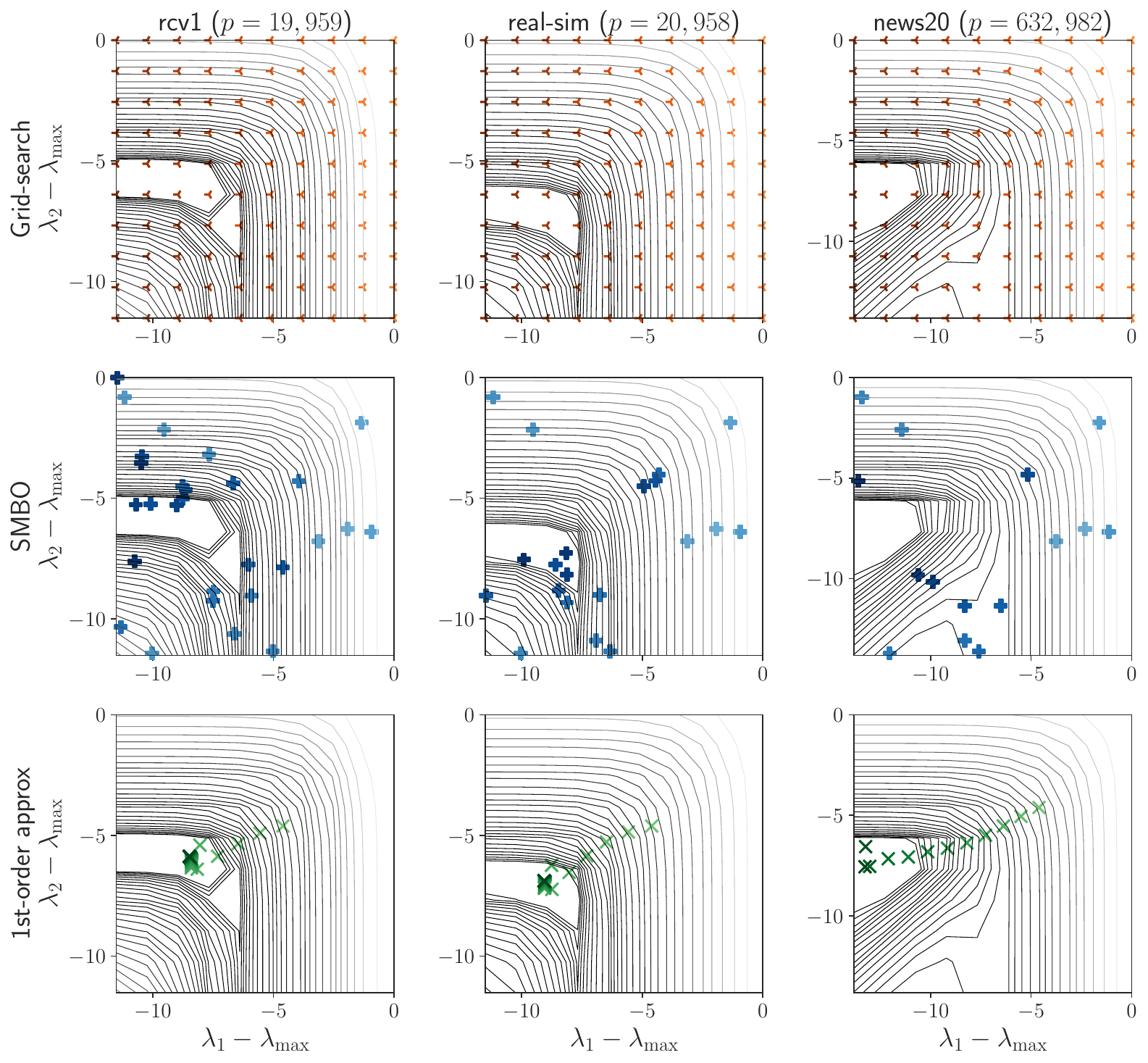}
        \includegraphics[width=\figsize\linewidth]{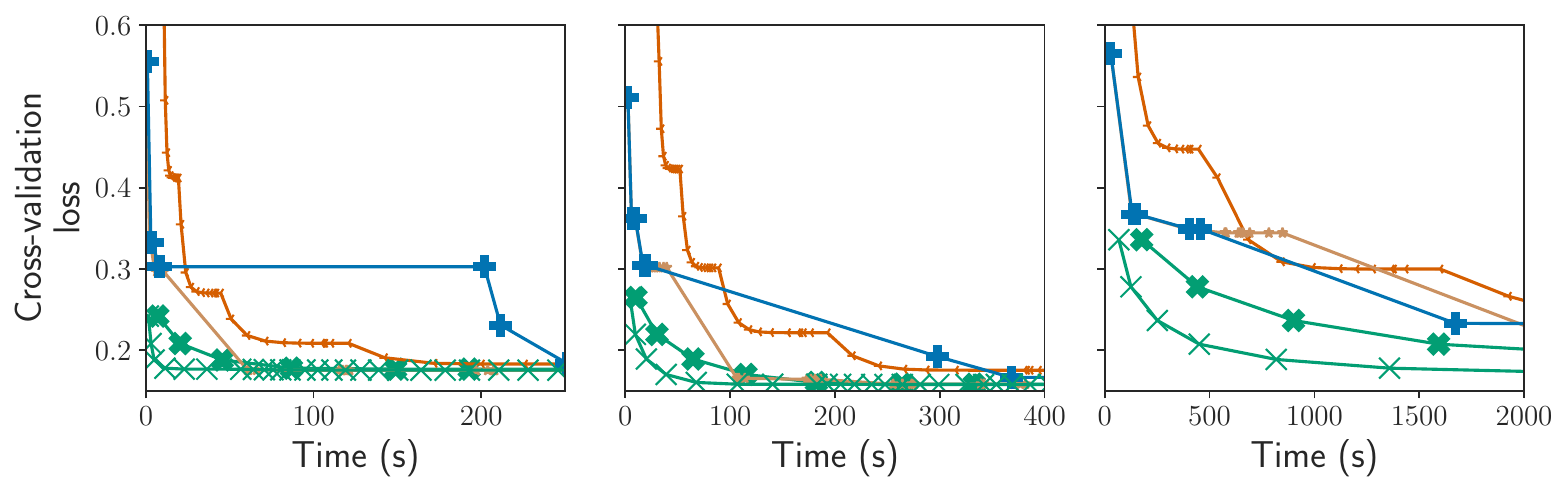}
     \end{subfigure}
    \caption{
        \textbf{Elastic net cross-validation, time comparison ($2$ hyperparameters).}
        Level sets of the cross-validation loss (black lines, top) and cross-validation loss as a function of time (bottom) on
     \emph{rcv1}, \emph{real-sim} and \emph{news20} data sets.}
    \label{fig:ho_enet}
\end{figure}
%

%
\Cref{fig:ho_enet} represents the level sets of the cross-validation loss for the elastic net 
(three top rows) and the cross-validation loss as a function of time (bottom).
One can see that after $5$ iterations the SMBO algorithm (blue crosses) suddenly slows down (bottom) as the hyperparameter suggested by the algorithm leads to a costly optimization problem to solve, while first-order methods converge quickly as for Lasso CV. In the present context, inner problems are slower to solve for low values of the regularization parameters.

\textit{Multiclass sparse logistic regression ($\#$ classes hyperparameters, \Cref{fig:ho_multiclass_logreg}).}
We consider a multiclass classification problem with $q$ classes.
The design matrix is noted $X \in \bbR^{n \times p}$, and the target variable
$y \in \{1, \dots, q \}^n$.
We chose to use a one-versus-all model with $q$ regularization parameters.
%
%
We use a binary cross-entropy for the inner loss
\begin{align}
    \psi^k (\beta, \lambda_k ; X, y)
    \eqdef
    -\frac{1}{n}\sum_{i=1}^{n}
    \left (
        \mathbbm{1}_{y_i = k}
        \ln (\sigma(X_{i:} \beta)) +
        (1 - \mathbbm{1}_{y_i = k})
        \ln(1-\sigma(X_{i:} \beta))
    \right )
     + e^{\lambda_k} \normin{\beta}_1
     \enspace ,
    \nonumber
\end{align}
and a multiclass cross-entropy for the outer criterion
\begin{align}\label{eq:criterion_multiclass}
    \cC
    \left (
        \hat \beta^{(\lambda_1)},
        \dots,
        \hat \beta^{(\lambda_q)}
        ; X, y
    \right)
    \eqdef
    - \sum_{i=1}^n \sum_{k=1}^q \ln
    \left (
        \frac{e^{X_{i:} \hat \beta^{(\lambda_k)}}}{
            \sum_{l=1}^q e^{X_{i:} \hat \beta^{(\lambda_l)}}
        }
    \right) \mathbbm{1}_{y_i = k}
    \enspace .
\end{align}
With a single train/test split, the bilevel problem to solve writes:
\begin{equation}\label[pb_multiline]{pb:bilevel_opt_multiclass}
    \begin{aligned}
    \argmin_{\lambda \eqdef (\lambda_1, \dots, \lambda_q) \in \bbR^q}
    &
    \cC
        \left (
            \hat \beta^{(\lambda_1)},
            \dots,
            \hat \beta^{(\lambda_q)} ;
            X^\mathrm{test}, y^\mathrm{test}
        \right)
    \\
    \st
    &\hat \beta^{(\lambda_k)} \in \argmin_{\beta \in \bbR^p}
      \psi^k(\beta, \lambda_k ; X^\mathrm{train}, y^\mathrm{train} )    \quad \forall k \in [q]
     \enspace.
    \end{aligned}
\end{equation}
%
%
\begin{figure*}[tb]
        \centering
        \includegraphics[width=0.62\linewidth]{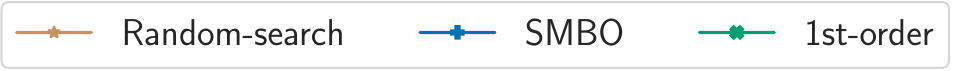}
        \includegraphics[width=\figsize\linewidth]{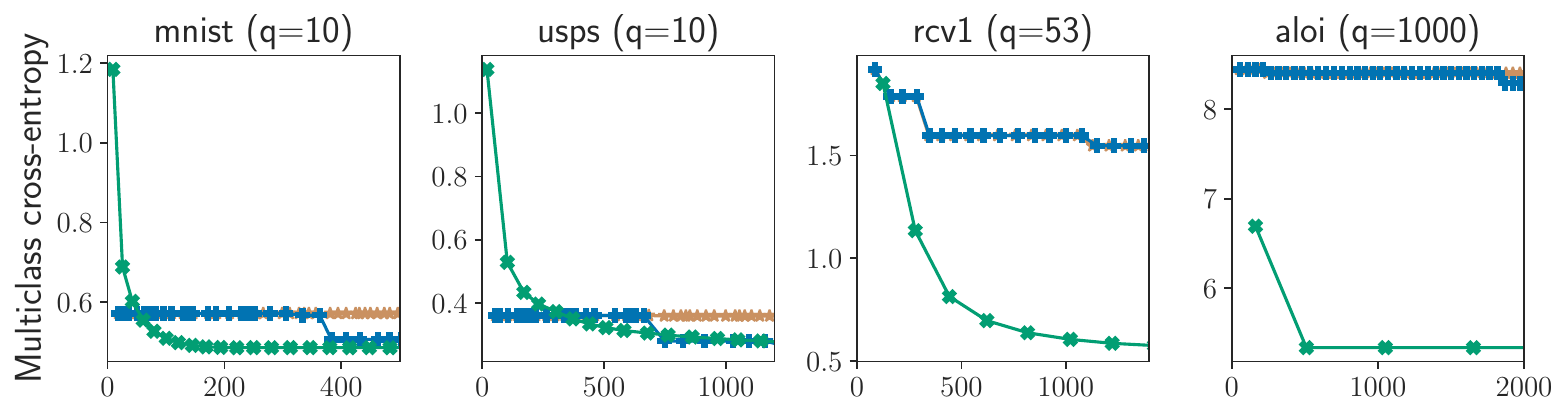}
        \includegraphics[width=\figsize\linewidth]{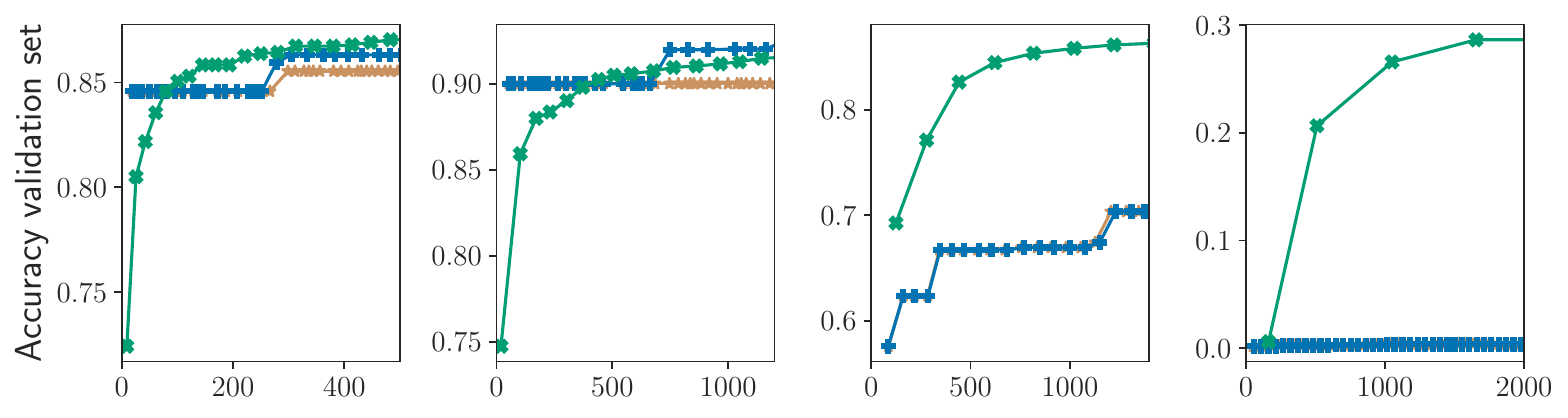}
        \includegraphics[width=\figsize\linewidth]{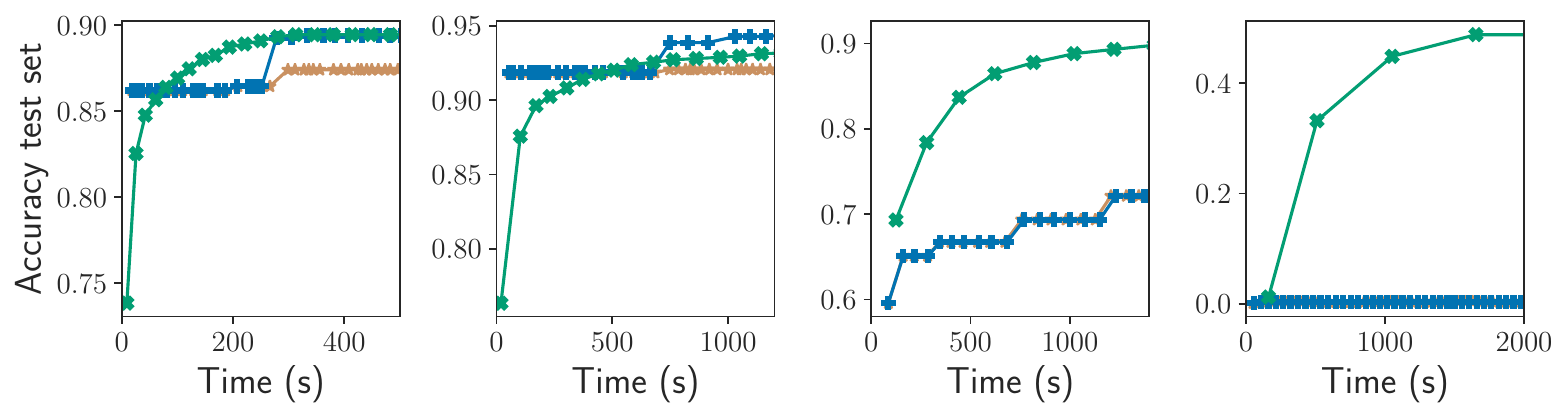}
    \caption{
        \textbf{Multiclass sparse logistic regression hold-out, time comparison ($\#$ classes hyperparameters).}
        Multiclass cross-entropy (top), accuracy on the validation set (middle), and accuracy on the test set (bottom) as a function of time on \emph{mnist}, \emph{usps} ($q=10$ classes), \emph{rcv1} ($q=53$ classes), \emph{aloi} ($q=1000$ classes).}
    \label{fig:ho_multiclass_logreg}
\end{figure*}
%
\Cref{fig:ho_multiclass_logreg} represents the multiclass cross-entropy (top), the accuracy on the validation set (middle) and the accuracy on the test set (unseen data, bottom).
When the number of hyperparameter is moderate ($q=10$, on \emph{mnist} and \emph{usps}), the multiclass cross-entropy reached by SMBO and random techniques is as good as first-order techniques.
This is expected and follows the same conclusion as \citet{Bergstra_Bengio12,Frazier2018}: when the number of hyperparameters is moderate, SMBO and random techniques can be used efficiently. %
However, when the number of hyperparameters increases (\emph{rcv1}, $q=53$ and \emph{aloi}, $q=1000$), the hyperparameter space is too large: zero-order solvers simply fail.
On the contrary, first-order techniques manage to find hyperparameters leading to significantly better accuracy.

    \begin{remark}
        On the data used in \Cref{fig:ho_multiclass_logreg}, the model with one hyperparameter per class did not yield significantly better test accuracy compared to a multiclass logistic regression with only one regularization hyperparameter for all the classes.
        This may mean that the model with one hyperparameter per class is not well suited for this data.
        It can also be due to the fact that in this case, the bilevel optimization problem becomes highly non-convex, and only converges toward a poor local minima.
        We want to emphasize that we provide an efficient way to compute the hypergradient $\nabla_\lambda \cL(\lambda)$. Besides, to our knowledge, the perfect resolution of the full bilevel optimization problem with a non-smooth inner problem remains an open question.
    \end{remark}

%
\section{Conclusion}
In this work we considered the problem of hyperparameter optimization to select the regularization parameter of linear models with non-smooth objective.
Casting this problem as a bilevel optimization problem, we proposed to use first-order methods.
We showed that the usual automatic differentiation techniques, implicit differentiation, \forwardandbackward, can be used to compute the hypergradient, despite the non-smoothness of the inner problem.
Experimentally, we showed the interest of first-order techniques to solve bilevel optimization on a wide range of estimators ($\ell_1$ penalized methods, SVM, etc.) and data sets.
The presented techniques could also be extended to more general bilevel optimization problems, in particular implicit differentiation could be well suited for meta-learning problems, with a potentially large number of hyperparameters.
Another important future direction would be to extend the work on stochastic hypergradients \citep{Grazzi_Pontil_Salzo2021} in the non-smooth case.

\acks{This work was partially funded by the ERC Starting Grant SLAB ERC-StG-676943, the ANR BrAIN ANR-20-CHIA-0016, the ANR CaMeLOt ANR-20-CHIA-0001-01, and the ANR grant GraVa ANR-18-CE40-0005.
Part of this work has been carried out at the Machine Learning Genoa (MaLGa) center, Università di Genova (IT).
M. M. acknowledges the financial support of the European Research Council (grant SLING 819789).
}

\clearpage
\bibliography{references_all}
\clearpage
\renewcommand{\theHsection}{A\arabic{section}}  
\appendix
\onecolumn
\section{Implicit Differentiation Examples}
\label{app:algos_implicit_diff}
For exposition purpose we provide instantiations of the \implicitfull (\Cref{alg:implicit}) for multiple optimization problems.
For the Lasso (\Cref{alg:implicit_lasso})
\begin{align*}
    \hat \beta
    \in
    \argmin_{\beta \in \bbR^p}
    \frac{1}{n} \normin{y - X \beta}^2
    +
    e^\lambda \normin{\beta}_1
    \enspace ,
\end{align*}
the elastic net (\Cref{alg:implicit_enet})
\begin{align*}
    \hat \beta
    \in
    \argmin_{\beta \in \bbR^p}
    \frac{1}{n} \normin{y - X \beta}^2
    +
    e^{\lambda_1} \normin{\beta}_1
    +
    e^{\lambda_2} \normin{\beta}_2^2
    \enspace ,
\end{align*}
the weighted Lasso (\Cref{alg:implicit_wlasso})
\begin{align*}
    \hat \beta
    \in
    \argmin_{\beta \in \bbR^p}
    \frac{1}{n} \normin{y - X \beta}^2
    +
    \sum_{j=1}^p e^{ \lambda_j} |\beta_j|
    \enspace ,
\end{align*}
and the dual of the SVM
\begin{align*}
    \hat w
    \in
    \argmin_{w \in \bbR^n}
    \frac{1}{2}
    w^\top (y \odot X) (y \odot X)^\top w^\top
    - w^\top \ind
    +
    \sum_{i=1}^n \iota_{0 \leq w_i \leq e^\lambda}
    \enspace .
\end{align*}
%

%
\begin{figure}[H]
    \begin{minipage}[t]{0.5\linewidth}

{\fontsize{5}{4}\selectfont
\begin{algorithm}[H]
\SetKwInOut{Input}{input}
\SetKwInOut{Init}{init}
\SetKwInOut{Parameter}{param}
\caption{\textsc{Lasso \implicitfull}}
\Input{
$
\lambda \in \bbR,
\epsilon > 0
$}

\tcp*[l]{compute the solution of inner problem}
Find $\beta$ such that:
$\Phi(\beta, \lambda) - \Phi(\hat \beta, \lambda) \leq \epsilon$

\tcp*[l]{compute the gradient}

$S = \condset{j \in [p]}{\beta_j \neq 0 }$;
$A = \frac{1}{n} X_{:, S}^\top X_{:, S} $

Find $v\in\bbR^{|S|}$ s.t.
$\normin{A^{-1} \nabla_S \cC(\beta) - v} \leq \epsilon$

$B = - e^\lambda \sign(\beta_S) \in \bbR^{|S|}$

$\nabla \cL(\lambda)
= v^\top B \in \bbR
$

\Return{
    $\cL(\lambda) \eqdef \cC(\beta), \nabla \cL(\lambda)$
    }
\label{alg:implicit_lasso}
\end{algorithm}
}
    \end{minipage}\hfill
    \begin{minipage}[t]{0.49\linewidth}
        {\fontsize{5}{4}\selectfont
        \begin{algorithm}[H]
        \SetKwInOut{Input}{input}
        \SetKwInOut{Init}{init}
        \SetKwInOut{Parameter}{param}
        \caption{\textsc{Elastic net \implicitfull}}
        \Input{
        $
        \lambda_1, \lambda_2 \in \bbR,
        \epsilon > 0
        $}

        \tcp*[l]{compute the solution of inner problem}
        Find $\beta$ such that:
        $\Phi(\beta, \lambda) - \Phi(\hat \beta, \lambda) \leq \epsilon$

        \tcp*[l]{compute the gradient}

        $S = \condset{j \in [p]}{\beta_j \neq 0 }$;
        $A = \frac{1}{n} X_{:, S}^\top X_{:, S} $

        Find $v\in\bbR^{|S|}$ s.t.
        $\normin{A^{-1} \nabla_S \cC(\beta) - v} \leq \epsilon$

        $B = - [e^{\lambda_1} \sign(\beta_{S}), e^{\lambda_2} \beta_S] \in \bbR^{|S| \times 2}$

        $\nabla \cL(\lambda)
        = v^\top B \in \bbR^2
        $

        \Return{
            $\cL(\lambda) \eqdef \cC(\beta), \nabla \cL(\lambda)$
            }
        \label{alg:implicit_enet}
        \end{algorithm}
        }
    \end{minipage}

    \begin{minipage}[t]{0.5\linewidth}

{\fontsize{5}{4}\selectfont
\begin{algorithm}[H]
\SetKwInOut{Input}{input}
\SetKwInOut{Init}{init}
\SetKwInOut{Parameter}{param}
\caption{\textsc{Weighted Lasso \implicitfull}}
\Input{
$
\lambda \in \bbR^p,
\epsilon > 0
$}

\tcp*[l]{compute the solution of inner problem}
Find $\beta$ such that:
$\Phi(\beta, \lambda) - \Phi(\hat \beta, \lambda) \leq \epsilon$

\tcp*[l]{compute the gradient}

$S = \condset{j \in [p]}{\beta_j \neq 0 }$;
$A = \frac{1}{n} X_{:, S}^\top X_{:, S} $

Find $v\in\bbR^{|S|}$ s.t.
$\normin{A^{-1} \nabla_S \cC(\beta) - v} \leq \epsilon$

$B = - \diag(e^{\lambda_j} \sign(\beta_j))  \in \bbR^{|S| \times |S|}$

$\nabla \cL(\lambda)
= v^\top B \in \bbR^p
$

\Return{
    $\cL(\lambda) \eqdef \cC(\beta), \nabla \cL(\lambda)$
    }
\label{alg:implicit_wlasso}
\end{algorithm}
}
    \end{minipage}\hfill
    \begin{minipage}[t]{0.49\linewidth}
        {\fontsize{5}{4}\selectfont
        \begin{algorithm}[H]
        \SetKwInOut{Input}{input}
        \SetKwInOut{Init}{init}
        \SetKwInOut{Parameter}{param}
        \caption{\textsc{SVM dual \implicitfull}}
        \Input{
        $
        \lambda \in \bbR,
        \epsilon > 0
        $}
        \Init{$\jac = 0_{\bbR^p}$}

        \tcp*[l]{compute the solution of inner problem}
        Find $w$ such that:
        $\Phi(w, \lambda) - \Phi(\hat w, \lambda) \leq \epsilon$

        \tcp*[l]{compute the gradient}

        $S_{0} \eqdef \condset{i \in [n]}{w_i = 0}$;
        $\jac_{S_{0}} = 0$

        $S_{\lambda} \eqdef \condset{i \in [n]}{w_i = e^\lambda}$;
        $\jac_{S_{\lambda}} = e^\lambda$

        $S = \condset{i \in [n]}{w_i \neq 0 \text{ and } w_i \neq e^\lambda}$

        $A =(y \odot X)_{S :} (y \odot X)_{S :}^\top $

        Find $v\in\bbR^{|S|}$ s.t.
        $\normin{A^{-1 \top } \nabla_S \cC(w) - v} \leq \epsilon$

        $B =
        (y \odot X)_{S :} (y \odot X)_{S_\lambda :}^\top
       \jac_{ S_\lambda}$

        $\nabla \cL(\lambda)
        = \jac_{S_{\lambda} :}^\top \nabla_{S_{\lambda}} \cC(\beta)
        + v^\top B$

        \Return{
            $\cL(\lambda) \eqdef \cC(\beta), \nabla \cL(\lambda)$
            }
        \label{alg:implicit_svm}
        \end{algorithm}
        }
    \end{minipage}
\end{figure}

\section{Additional Lemmas}
\label{app:sec_preliminaries}
%
\subsection{Differentiability of the Proximal Operator}
Here we recall results on the differentiability of the proximal operator at the optimum.
%
\begin{lemma}[\citealt{Klopfenstein_Bertrand_Gramfort_Salmon_Vaiter20}, Lemmas 2 and 3]\label{lemma:diff_prox}
    Let $0 < \gamma_j \leq 1 / L_j$.
    Let $\lambda \in \bbR^r$ and $\Lambda$ a neighborhood of $\lambda$.
    Consider a solution $\hat \beta\in \argmin_{\beta\in\bbR^p} \Phi(\beta, \lambda)$ and $\hat{S}$ its generalized support. Suppose
    \begin{enumerate}\setlength\itemsep{1pt}
        \item \Cref{ass:smoothness,ass:proper,ass:taylor_expansion} hold.
        \item \Cref{ass:non_degeneracy} hold on $\Lambda$.
    \end{enumerate}
    Then, for all $j\in \hat{S}$, the map $\beta \mapsto \prox_{\gamma_j g_j(\cdot, \lambda)}$ is differentiable at $\hat{\beta}_{\hat{S}}$.
    Moreover, for all $j \in \hat{S}^c$, $\prox_{\gamma_j g_j(\cdot, \lambda)}$ is constant around $\hat{\beta}_j-\gamma_j \nabla_j f(\hat{\beta})$. Thus, $\beta \mapsto \prox_{\gamma_j g_j(\cdot, \lambda)}(\beta_j- \gamma_j \nabla_j f(\beta))$ is differentiable at $\hat{\beta}$ with gradient $0$.
\end{lemma}
%

\subsection{Linear Convergence}
We now detail the following result: an asymptotic vector autoregressive sequence, with an error term vanishing linearly to $0$, converges linearly to its limit.
In a more formal way:
\begin{lemma}\label{lemma:asymptotic_var}
    Let $A \in \bbR^{p \times p}, \textcolor{blue}{b \in \bbR^p}$ with $\rho(A) < 1$.
    Let $(\jac^{(k)})_{k \in \bbN}$ be a sequence of $\bbR^p$ such that
    \begin{equation}\label{eq:recursion_jac}
        \jac^{(k+1)} = A \jac^{(k)} + b + \epsilon^{(k)} \enspace ,
    \end{equation}
    with $(\epsilon^{(k)})_{k \in \bbN}$ a sequence which converges linearly to $0$,
    then $(\jac^{(k)})_{k \in \bbN}$ converges linearly to its limit $\hat \jac \eqdef  (\Id - A)^{-1} b$.
\end{lemma}
\begin{proof}
    Assume $(\epsilon^{(k)})_{k \in \bbN}$ converges linearly. Then, there exists
    $c_1 > 0, 0 < \nu < 1$ such that
    \begin{equation}
        \normin{\epsilon^{(k)}} \leq c_1 \nu^k
        \enspace .
        \nonumber
    \end{equation}
   Applying a standard result on spectral norms (see \citealt[Chapter 2, Lemma 1]{Polyak87}) yields a bound on $\normin{A^k}_2$.
   More precisely, for every $\delta > 0$ there is a constant $c_2(\delta) = c_2$ such that
    \begin{align}
        \normin{A^k}_2 \leq c_2 (\rho(A) + \delta)^k
        \enspace .
        \nonumber
    \end{align}
    Without loss of generality, we consider from now on a choice of $\delta$ such that $\rho(A) + \delta<1$.
    Since $\hat \jac = (\Id - A)^{-1} b$ the limit $\hat \jac$ of the sequence satisfies
    \begin{equation}
        \hat \jac = A \hat \jac + b \enspace. \label{eq:limit_jac}
    \end{equation}
    Taking the difference between \Cref{eq:recursion_jac,eq:limit_jac} yields:
    \begin{equation} \label{eq:recursion_diff}
        \jac^{(k+1)} - \hat \jac
        = A (\jac^{(k)} - \hat \jac )
          + \epsilon^{(k)} \enspace .
    \end{equation}
    Unrolling \Cref{eq:recursion_diff} yields $\jac^{(k+1)} - \hat \jac =
        A^{k+1} (\jac^{(0)} - \jac) + \sum_{k'=0}^k A^{k'} \epsilon^{(k-k')}$. Taking the norm on both sides and using the triangle inequality leads to
   \allowdisplaybreaks
    \begin{align*}
        \normin{\jac^{(k+1)} - \hat \jac}_2
        \leq &
        ~\normin{A^{k+1} (\jac^{(0)} - \jac)}_2
        + \sum_{k'=0}^k \normin{ A^{k'} }_2 \normin{ \epsilon^{(k-k')} }
        \\
        \leq &
        ~\normin{A^{k+1}}_2 \cdot \normin{\jac^{(0)} - \hat \jac}_2
        + c_1 \sum_{k'=0}^k \normin{A^{k'}}_2 \cdot \nu^{k-k'}
        \\
        \leq &
        ~c_2(\rho(A) + \delta)^{k+1} \cdot \normin{\jac^{(0)} - \hat \jac}_2
        + c_1 \sum_{k'=0}^k  c_2(\rho(A) + \delta)^{k'} \nu^{k-k'}
    \end{align*}
    We can now split the last summand in two parts and obtain the following bound, reminding that $\rho(A) + \delta<1$
    \begin{align*}
        \normin{\jac^{(k+1)} - \hat \jac}_2
        \leq &
        ~c_2(\rho(A) + \delta)^{k+1}
        \cdot \normin{\jac^{(0)} - \hat \jac}_2\\
        & + c_1 c_2 \left(\sum_{k'=0}^{k/2} (\rho(A) + \delta)^{k'} \nu^{k-k'}
        + \sum_{k'=k/2}^{k} (\rho(A) + \delta)^{k'} \nu^{k-k'} \right)\\
        \leq &  ~c_2(\rho(A) + \delta)^{k+1}  \cdot \normin{\jac^{(0)} - \hat \jac}_2 + \frac{c_1 c_2 (\rho(A) + \delta)}{1 - \rho(A) - \delta} \sqrt{\nu}^{k}\\
        & + \frac{c_1 c_2 \nu}{1 - \nu}  \sqrt{(\rho(A) + \delta)}^{k} \enspace .
    \end{align*}
    Thus, $(\jac^{(k)})_{k \in \mathbb{N}}$ converges linearly towards its limit $\hat{\jac}$.
\end{proof}
%
\section{Proof of \Cref{thm:iterdiff_linear_convergence}}
\label{app:proof:nonsmooth_iterdiff}
%
\iterdifflinconv*

\begin{proof}
    We first prove \Cref{thm:iterdiff_linear_convergence} for proximal gradient descent.

    \textit{Proximal gradient descent case.}
    Solving \Cref{pb:generic_inner_pb} with proximal gradient descent leads to the following updates:
    \begin{align}\label{eq:forward_backward_update}
        \beta^{(k+1)}
        =
        \prox_{\gamma g(\cdot, \lambda)}
        (\underbrace{\beta^{(k)} - \gamma \nabla f(\beta^{(k)})}_{z^{(k)}})
        \enspace .
    \end{align}
    Consider the following sequence $(\jac^{(k)})_{k \in \bbN}$ defined by
    \begin{align}\label{eq:forward_backward_diff}
        \jac^{(k+1)}
        =
        \partial_z \prox_{\gamma g(\cdot, \lambda)}(z^{(k)}) \odot
        \left(\Id - \gamma \nabla^2 f(\beta^{(k)}) \right)
        \jac^{(k)}
        + \partial_{\lambda} \prox_{\gamma g(\cdot, \lambda)}(z^{(k)}) \enspace .
    \end{align}
    Note that if $\prox_{\gamma g(\cdot, \lambda)}$ is not differentiable with respect to the first variable at $z^{(k)}$ (respectively with respect to the second variable $\lambda$), any weak Jacobian can be used.
    When \ref{hyp:prox_c1} holds, differentiating \Cref{eq:forward_backward_update} \wrt $\lambda$ yields exactly
    \Cref{eq:forward_backward_diff}.

    \Cref{ass:smoothness,ass:proper,ass:non_degeneracy,ass:taylor_expansion} and the convergence of $(\beta^{(k)})$ toward $\hat \beta$ ensure proximal gradient descent algorithm has finite identification property \citep[Thm. 3.1]{Liang_Fadili_Peyere14}: we note $K$ the iteration when identification is achieved.
    As before, the separability of $g$, \Cref{ass:smoothness,ass:proper,ass:taylor_expansion,ass:non_degeneracy} ensure (see \Cref{{lemma:diff_prox}})
    $\partial_z \prox_{\gamma g(\cdot, \lambda)}(z^k)_{\hat{S}^c} = 0$, for all $k\geq K$.
    Thus, for all $k\geq K$,
    \begin{align*}
    \jac^{(k)}_{\hat{S}^c :}
    =
    \hat{\jac}_{\hat{S}^c :}
    = \partial_{\lambda} \prox_{\gamma g(\cdot, \lambda)}(z^{(k)})_{\hat{S}^c :}
    \enspace .
    \end{align*}
    %
    The updates of the Jacobian then become
    \begin{align*}
        \jac^{(k+1)}_{\hat{S} :}
        =
        \partial_z \prox_{\gamma g(\cdot, \lambda)}(z^{(k)})_{\hat{S}} \odot
        \left(
            \Id - \gamma  \nabla^2_{\hat S, \hat S} f(\beta^{(k)}) \right)
        \jac_{\hat{S} :}^{(k)}
        + \partial_{\lambda} \prox_{\gamma g(\cdot, \lambda)}(z^{(k)})_{\hat{S} :} \enspace .
    \end{align*}
    From \Cref{ass:taylor_expansion}, we have that $f$ is locally $\mathcal{C}^3$ at $\hat{\beta}$, $g( \cdot, \lambda)$ is locally $\mathcal{C}^2$ at $\hat{\beta}$ hence
    $\prox_{g( \cdot, \lambda)}$ is locally $\mathcal{C}^2$.
    The function $ \beta \mapsto \partial_z \prox_{\gamma g(\cdot, \lambda)}(\beta - \gamma \nabla f(\beta))_{\hat{S}} \odot
    (\Id - \gamma  \nabla^2_{\hat S, \hat S} f(\beta) )$
    is differentiable at $\hat \beta$.
    Using \ref{hyp:prox_lip} we have that
    $\beta \mapsto \partial_{\lambda} \prox_{\gamma g(\cdot, \lambda)}( \beta - \gamma \nabla f(\beta) )_{\hat{S} :}$ is also differentiable at $\hat \beta$.
    Using the Taylor expansion of the previous functions yields
    \begin{align}\label{eq:taylor_PGD}
        \jac^{(k+1)}_{\hat{S} :}
        =
        \underbrace{\partial_z \prox_{\gamma g(\cdot, \lambda)}(\hat{z})_{\hat{S}}
        \odot
        \left(
            \Id - \gamma \nabla^2_{\hat S, \hat S} f(\hat{\beta}) \right)}_{A}
        \jac_{\hat{S} :}^{(k)}
        +
        \underbrace{
            \partial_{\lambda} \prox_{\gamma g(\cdot, \lambda)}(\hat{z})_{\hat{S} :}}_{b}
            + \underbrace{o(\normin{\beta^{(k)} - \hat{\beta}})}_{\epsilon^{(k)}} \enspace .
    \end{align}
    Thus, for
    $0 < \gamma \leq 1 / L$,
    \begin{align}
        \label{eq:spec_A}
        \rho(A)
        \leq \normin{A}_2
        & \leq
        \underbrace{\normin{\partial_z \prox_{\gamma g(\cdot, \lambda)}(\hat{z})_{\hat{S}}}}_{
            \leq 1 \text{ (non-expansiveness)}}
        \,\cdot \hspace{-1.5cm} \underbrace{
            \normin{\Id - \gamma\nabla_{\hat S, \hat S}^2 f(\hat{\beta}) }_2}_{
                \quad \quad\quad \quad <1 \text{ (\Cref{ass:restricted_injectivity} and $0 < \gamma \leq 1 / L$)}}
        < 1 \enspace .
    \end{align}
    The inequality on the derivative of the proximal operator comes from the non-expansiveness of proximal operators.
    The second inequality comes from \Cref{ass:restricted_injectivity} and $0 < \gamma \leq 1 / L$.

    \Cref{ass:smoothness,ass:proper,ass:non_degeneracy,ass:taylor_expansion,ass:restricted_injectivity} and the convergence of $(\beta^{(k)})$ toward $\hat \beta$ ensure
    $(\beta^{(k)})_{k \in \bbN}$ converges locally linearly \citep[Thm. 3.1]{Liang_Fadili_Peyere14}.
    The asymptotic autoregressive sequence in \Cref{eq:taylor_PGD},  $\rho(A) < 1$, and  the local linear convergence of $(\epsilon^{(k)})_{k \in \bbN}$, yield our result using \Cref{lemma:asymptotic_var}.

    We now prove \Cref{thm:iterdiff_linear_convergence} for proximal coordinate descent.

    \textit{Proximal coordinate descent.}
    Compared to proximal gradient descent,
    the analysis of coordinate descent requires studying functions defined as a the composition of $p$ applications, each of them only modifying one coordinate.

    Coordinate descent updates read as follows
    \begin{align}\label{eq:CD_update}
        \beta_j^{(k, j)}
        =
        \prox_{\gamma_j g_j(\cdot, \lambda)}
        \underbrace{\left(
            \beta_j^{(k,j-1)} - \gamma_j  \nabla_j f(\beta^{(k, j-1)})
        \right)}_{\eqdef z_j^{(k, j-1)} }
        \enspace .
    \end{align}
    We consider the following sequence
    \begin{align}\label{eq:CD_jac_update}
        \jac^{(k, j)}_{j:}
         = ~ &
        \partial_z \prox_{\gamma_j g_j(\cdot, \lambda)}
        ( z_j^{(k, j-1)} )
        \left(
            \jac^{(k, j-1)}_{j:}
            - \gamma_j  \nabla^2_{j:} f(\beta^{(k,j-1)})  \jac^{(k, j-1)}
        \right)
        \nonumber\\
        & + \partial_{\lambda} \prox_{\gamma_j g_j(\cdot, \lambda)}
        (z_j^{(k,j-1)}) \enspace .
    \end{align}
    Note that if $\prox_{\gamma g(\cdot, \lambda)}$ is not differentiable with respect to the first variable at $z^{(k)}$ (respectively with respect to the second variable $\lambda$), any weak Jacobian can be used.
    When \ref{hyp:prox_c1} holds, differentiating \Cref{eq:CD_update} \wrt $\lambda$ yields exactly
    \Cref{eq:CD_jac_update}.

    \Cref{ass:proper,ass:smoothness,ass:non_degeneracy,ass:taylor_expansion} and the convergence of $(\beta^{(k)})_{k \in \bbN}$ toward $\hat \beta$ ensure
    proximal coordinate descent has finite identification property \citep[Thm. 1]{Klopfenstein_Bertrand_Gramfort_Salmon_Vaiter20}: we note $K$ the iteration when identification is achieved.
    Once the generalized support $\hat{S}$ (of cardinality $\hat{s}$) has been identified, we have that for all $k\geq K$, $\beta_{\hat{S}^c}^{(k)} =\hat{\beta}_{\hat{S}}$ and for any $j \in \hat{S}^c$, $\partial_z \prox_{\gamma_j g_j(\cdot, \lambda)}(z_{j}^{(k, j-1)})=0$.
    Thus $\jac_{j:}^{(k, j)} = \partial_{\lambda}\prox_{\gamma_j g_j(\cdot, \lambda)}(z_j^{(k, j-1)})$.
    Then, we have that for any $j\in \hat{S}$ and for all $k\geq K$:
    \begin{align}
        \jac^{(k, j)}_{j:}
        & =
        \partial_z \prox_{\gamma_j g_j(\cdot, \lambda)}
        (z_j^{(k, j-1)})
        \left(
            \jac^{(k, j-1)}_{j:}
            - \gamma_j
            \nabla^2_{j, \hat S} f(\beta^{(k,j-1)}) \jac^{(k, j-1)}_{\hat{S}:}
        \right)
        \nonumber\\
        & \hspace{1em}
        + \partial_{\lambda} \prox_{\gamma_j g_j(\cdot, \lambda)}
        ( z_j^{(k,j-1)})
        - \gamma_j \partial_z \prox_{\gamma_j g_j(\cdot, \lambda)}
        (z_j^{(k, j-1)})
         \nabla^2_{j, \hat S^c} f(\beta^{(k,j-1)}) \jac^{(k, j-1)}_{\hat{S}^c :}
        \enspace .
        \nonumber
    \end{align}
    Let $e_1, \dots, e_{\hat s}$ be the vectors of the canonical basis of $\bbR^{\hat s}$.
    We can consider the applications
    \begin{align}
        \bbR^p & \rightarrow \bbR^{\hat s}
        \nonumber
        \\
        \beta & \mapsto
         \partial_z \prox_{\gamma_j g_j(\cdot, \lambda)}
        \left(
            \beta_j - \gamma_j \nabla_j f(\beta)
        \right)
        \left(
            e_j - \gamma_j \nabla^2_{j, \hat S} f(\beta)
        \right)
        \enspace ,
        \nonumber
    \end{align}
    and
    \begin{align*}
        \bbR^p & \rightarrow \bbR^{\hat s \times r}
        \\
        \beta  & \mapsto
        \partial_{\lambda} \prox_{\gamma_j g_j(\cdot, \lambda)}
        \left(
            \beta_j -\gamma_j \nabla_j f(\beta)
        \right)
       - \gamma_j\partial_z \prox_{\gamma_j g_j(\cdot, \lambda)}
        \left(
            \beta_j -\gamma_j \nabla_j f(\beta)
        \right)
        \nabla^2_{j, \hat S^c} f(\beta) \hat{\jac}_{\hat{S}^c:}
        \enspace ,
    \end{align*}
    which are both differentiable at $\hat \beta$ using \Cref{ass:taylor_expansion} and \ref{hyp:prox_lip}.
    The Taylor expansion of the previous functions yields:
    \begin{align}
        \label{eq:expansion_process}
        \jac^{(k, j)}_{j:}
        & =
        \partial_z \prox_{\gamma_j g_j(\cdot, \lambda)}
        \left(
            \hat{z}_j
        \right)
        \left(
            e_j - \gamma_j \nabla^2_{j,\hat S} f(\hat{\beta})
        \right)
        \jac^{(k, j-1)}_{\hat{S} :}
        \nonumber\\
        & \hspace{1em}
        + \partial_{\lambda} \prox_{\gamma_j g_j(\cdot, \lambda)}
        \left(
            \hat{z}_j
        \right)
        - \gamma_j \partial_z \prox_{\gamma_j g_j(\cdot, \lambda)}
        \left(
            \hat{z}_j
        \right)
        \nabla^2_{j, \hat S^c} f(\hat{\beta})
        \jac^{(k, j-1)}_{\hat{S}^c :}\nonumber\\
        & \hspace{1em} + o(||\beta^{(k, j-1)} - \hat{\beta}||)
        \enspace .
        \nonumber
    \end{align}
    Let $j_1, \dots, j_{\hat s}$ be the indices of the generalized support of $\hat \beta$.
    When considering a full epoch of coordinate descent, the Jacobian is obtained as the product of matrices of the form
    \begin{align}
        A_{s}^\top =
        \left (
            \begin{array}{c|c|c|c|c|c|c}
                e_1 & \hdots & e_{s-1} & v_{j_s} & e_{s+1} & \hdots & e_{\hat{s}}
            \end{array}
        \right)
        \in \bbR^{\hat s \times \hat s}
        \enspace ,
        \nonumber
    \end{align}
    where $v_{j_s} = \partial_z \prox_{\gamma_{j_s} g_{j_s}}
    \left(
        \hat{z}_{j_s}
    \right)
    \left(
        e_s - \gamma_{j_s} \nabla^2_{j_s,\hat{S}}f(\hat{\beta})
    \right) \in \bbR^{\hat s}.$
    A full epoch can then be written
    \begin{align*}
        \jac^{(k+1)}_{\hat{S} :}
        =
        \underbrace{A_{\hat{s}} A_{\hat{s}-1} \hdots A_1}_{A} \jac^{(k)}_{\hat{S} :}
        + b
        + \epsilon^{(k)}
        \enspace ,
    \end{align*}
    for a certain $b \in \bbR^{\hat s}$.

    The spectral radius of $A$ is strictly bounded by $1$ \citep[Lemma 8]{Klopfenstein_Bertrand_Gramfort_Salmon_Vaiter20}: $\rho(A) < 1$.
    \Cref{ass:proper,ass:smoothness,ass:non_degeneracy,ass:taylor_expansion} and the convergence of $(\beta^{(k)})_{k \in \bbN}$ toward $\hat \beta$ ensure
    local linear convergence of $(\beta^{(k)})_{k \in \bbN}$
    \citep[Thm. 2]{Klopfenstein_Bertrand_Gramfort_Salmon_Vaiter20}.
   Hence, we can write the update for the Jacobian after an update of the coordinates from $1$ to $p$
    \begin{equation}
        \label{eq:varcd}
        \jac_{\hat S :}^{(k+1)}
        = A \jac_{\hat S :}^{(k)} + b + \epsilon^{(k)}
        \enspace ,
    \end{equation}
   with $(\epsilon^{(k)})_{k \in \bbN}$ converging linearly to 0.

    Recalling $\rho(A) < 1$, \Cref{lemma:asymptotic_var} and the last display yield our result using.
\end{proof}


\subsection{Proof of Theorem \ref{thm:approximate_gradient}
(Approximate
Hypergradients)}\label{app:proof_approx_grad}

\approxgrad*

\begin{proof}
    \textit{Overview of the proof.}
    Our goal is to bound the error between the approximate hypergradient $h$ returned by \Cref{alg:implicit} and the true hypergradient $\nabla \cL(\lambda)$.
    Following the analysis of \citet{Pedregosa16}, two sources of approximation errors arise when computing the hypergradient:
    \begin{itemize}
        \item Approximation errors from the inexact computation of $\hat \beta$.
        Dropping the dependency \wrt $\lambda$, we denote $\beta$ the approximate solution and suppose the problem is solved to precision $\epsilon$ with support identification \ref{hyp:espilon_sol}
        \begin{align*}
            \begin{cases}
            &\beta_{\hat S^c}
            = \hat \beta_{\hat S^c}
            \\
            &\normin{\beta_{\hat S} - \hat \beta_{\hat S} }
            \leq \epsilon
            \enspace .
            \end{cases}
        \end{align*}
        \item Approximation errors from the approximate resolution of the linear system, using \ref{hyp:espilon_sol} yields:
        \begin{align*}
            \normin{A^{-1 \top } \nabla_{\hat S} \cC(\beta) - v} \leq \epsilon
            \enspace .
        \end{align*}
        The exact solution of the exact linear system $\hat v$ satisfies
        \begin{align*}
            \hat v = \hat A^{-1 \top} \nabla_{\hat S} \cC(\hat \beta)
            \enspace ,
        \end{align*}
        with
        \begin{align}
            A
            &\eqdef
            \Id_{|\hat{S}|}-\underbrace{\partial_z\prox_{\gamma g(\cdot, \lambda)}\left (\beta-\gamma \nabla f(\beta)\right )_{\hat{S}}}_{\eqdef C}
            \underbrace{\left (\Id_{|\hat{S}|} - \gamma \nabla^2_{\hat{S}, \hat S} f(\beta) \right )}_{\eqdef  D}
            \nonumber
            \enspace , \\
            \hat{A}
            &\eqdef
            \Id_{|\hat{S}|}
            - \underbrace{
                \partial_z\prox_{\gamma g(\cdot, \lambda)}\left (
                    \hat{\beta}
                    - \gamma \nabla f(\hat{\beta} )
                \right )_{\hat{S}}}_{ \eqdef \hat{C} }\underbrace{
                \left (
                    \Id_{|\hat{S}|}
                    - \gamma \nabla^2_{\hat{S}, \hat S} f(\hat{\beta} )
                \right )}_{ \eqdef \hat{D} }
            \nonumber
            \enspace .
        \end{align}
    \item Using the last two points, the goal is to bound the difference between the exact hypergradient and the approximate hypergradient, $\normin{\nabla \mathcal{L}(\lambda) - h}$.
    Following \Cref{alg:implicit}, the exact hypergradient reads
    \begin{align*}
        \nabla \mathcal{L}(\lambda)
    =
    \hat{B}\hat{v}
    +  \hat{\jac}^\top_{\hat S^c :} \nabla_{\hat{S}^c}\mathcal{C}(\hat{\beta})
    \enspace ,
    \end{align*}
    and similarly for the approximate versions
    \begin{align*}
        h
    =
    B v
    +  \jac^\top_{\hat S^c :} \nabla_{\hat{S}^c}\mathcal{C}(\beta)
    \enspace ,
    \end{align*}
    with
    \begin{align}
        B
        & \eqdef
        \partial_{\lambda}\prox_{\gamma g(\cdot, \lambda)}
        \left (\beta
            - \gamma \nabla f(\beta ) \right )_{\hat{S} :}  \nonumber
              - \gamma  \partial_z\prox_{\gamma g(\cdot, \lambda)}\left (
                \beta
                 - \gamma \nabla f(\beta )
            \right )_{\hat{S}} \odot \left (
             \nabla^2_{\hat{S}, \hat S^c} f(\beta)
            \right ) \hat{\jac}_{\hat{S}^c :}
        \nonumber
       \\
        \hat{B}
        & \eqdef
        \partial_{\lambda}\prox_{\gamma g(\cdot, \lambda)}
        \left (
            \hat{\beta} -\gamma \nabla f(\hat{\beta} )
            \right )_{\hat{S} :} \nonumber
        \\
        & \quad   - \gamma  \partial_z\prox_{\gamma g(\cdot, \lambda)}\left (
                \hat{\beta}
                - \gamma \nabla f(\hat{\beta} )
            \right )_{\hat{S}} \odot \left (
             \nabla^2_{\hat{S}, \hat S^c} f(\hat{\beta} )
            \right ) \hat{\jac}_{\hat{S}^c :}
            \nonumber
            \enspace.
    \end{align}
    We can exploit these decompositions to bound the difference between the exact hypergradient and the approximate hypergradient
    \begin{align}
         \normin{\nabla \mathcal{L}(\lambda) - h}
        & =
        \normin{
            \hat{B}\hat{v}
            - B v
            +  \hat{\jac}^\top_{\hat S^c :} \nabla_{\hat{S}^c}\mathcal{C}(\hat{\beta})
            - \hat{\jac}^\top_{\hat S^c :} \nabla_{\hat{S}^c}\mathcal{C}(\beta)}
        \nonumber
        \\
        & \leq
        \normin{
            \hat{B}\hat{v}
            - Bv} +
            \normin{\hat{\jac}^\top_{\hat S^c :} \nabla_{\hat{S}^c}\mathcal{C}(\hat{\beta})
            - \hat{\jac}^\top_{\hat S^c :} \nabla_{\hat{S}^c}\mathcal{C}(\beta)}
        \nonumber
        \\
        & \leq
        \normin{
            \hat{B}\hat{v}
            - B\hat{v}
            + B\hat{v}
            - Bv}
            +  \normin{\hat{\jac}^\top_{\hat S^c :}(\nabla_{\hat{S}^c}\mathcal{C}(\hat{\beta}) - \nabla_{\hat{S}^c}\mathcal{C}(\beta))}
        \nonumber
        \\
        & \leq
        \normin{\hat{v}} \cdot \normin{\hat{B} - B}
            + \normin{B} \cdot \normin{\hat{v} - v}
            +  L_{\cC} \normin{\hat{\jac}^\top_{\hat S^c :}}  \cdot \normin{\beta - \hat{\beta}}
        \enspace . \label{eq:bound_hypergrad}
    \end{align}
\end{itemize}
    Bounding $\normin{\hat{v} - v}$ and $\normin{\hat{B} - B}$ in \Cref{eq:bound_hypergrad} yields the desired result which is bounding the difference between the exact hypergradient and the approximate hypergradient $\normin{\nabla \mathcal{L}(\lambda) - h}$.

\textit{Bound on $ \normin{\hat{v} - v}$.}
We first prove that $\normin{A - \hat{A}} = \bigo(\epsilon)$.
Let $L_H$ be the Lipschitz constant of the application $\beta \mapsto \nabla^2 f(\beta)$, then we have
    \begin{align}
        \normin{A - \hat{A}}_2
        & = \normin{CD - \hat{C}\hat{D}}_2
        \nonumber
        \\
        & \leq \normin{CD - C\hat{D}}_2
        + \normin{C\hat{D} - \hat{C}\hat{D}}_2
        \nonumber
        \\
        &
        \leq
        \underbrace{\norm{C}_2}_{\leq 1 \text{ (non-expansiveness)}}
        \underbrace{\normin{D - \hat{D}}_2 }_{ \leq L_H\normin{\beta - \hat{\beta}} \text{ using \ref{hyp:lip_Hessian}} }
        + \underbrace{\normin{\hat{D}}_2 }_{\leq 1}
        \underbrace{\normin{C-\hat{C}}_2 }_{
            \bigo (\normin{\beta - \hat{\beta}} )
            \text{ using \ref{hyp:prox_lip}} }
        \nonumber
        \\
        & \leq L_H \normin{\beta - \hat{\beta}} + \bigo (\normin{\beta - \hat{\beta}})
        \nonumber
        \\
        & = \bigo (\normin{\beta - \hat{\beta}})
        \label{eq:stability_A}
        \enspace .
    \end{align}
    Let $ \tilde{v}$ be the exact solution of the approximate system $ A^\top \tilde{v} \eqdef \nabla_{\hat{S}} \cC( \beta)$.
    The following conditions are met:
    \begin{itemize}
        \item $\hat{v}$ is the exact solution of the exact linear system and $\tilde{v}$ is the exact solution of the approximate linear system
        \begin{align*}
            \hat{A}^\top \hat{v}
            & \eqdef \nabla_{\hat{S}} \cC(\hat \beta)
            \\
            A^\top \tilde{v}
            & \eqdef \nabla_{\hat{S}} \cC( \beta)
            \enspace .
        \end{align*}
        \item One can control the difference between the exact matrix in the linear system $\hat A$ and the approximate matrix $A$
        \begin{align*}
            \normin{A - \hat{A}}_2
            \leq
            \delta \normin{\beta - \hat{\beta}}
            \enspace,
        \end{align*}
        for a certain $\delta > 0$ (\Cref{eq:stability_A}).
        \item One can control the difference between the two right-hand side of  the linear systems
        \begin{align*}
            \normin{\nabla_{\hat{S}} \cC(\beta) - \nabla_{\hat{S}} \cC(\hat \beta)}
            \leq
            L_{\cC} \normin{\beta - \hat{\beta}}
            \enspace ,
        \end{align*}
        since $\beta \mapsto \nabla \mathcal{C}(\beta)$ is $L_{\cC}$-Lipschitz continuous \ref{hyp:lip_criterion}.
        \item One can control the product of the perturbations
        \begin{align*}
            \delta \cdot \normin{\beta - \hat{\beta}}
            \cdot \normin{\hat{A}^{-1}}_2 \leq \rho < 1
            \enspace .
        \end{align*}
    \end{itemize}
    Conditions are met to apply the result by \citet[Thm 7.2]{Higham2002}, which leads to
    \begin{align}
        \normin{\tilde{v} -\hat{v}}
        & \leq
        \frac{\epsilon }{1-\epsilon  \normin{\hat{A}^{-1}} \delta}\left ( L_{\cC} \normin{\hat{A}^{-1}}
        +
        \normin{\hat{v}} \cdot \normin{\hat{A}^{-1}} \delta \right )
        \nonumber
        \\
        & \leq
        \frac{\epsilon}{1-\rho}
        \left (
            L_{\cC} \normin{\hat{A}^{-1}}
            + \normin{\hat{v}} \cdot  \normin{\hat{A}^{-1}} \delta
        \right )
        \nonumber
        \\
        & = \bigo(\epsilon)
        \label{eq:bound_v_v_tilde}
        \enspace .
    \end{align}
    The bound on $\normin{\tilde{v} -\hat{v}}$ finally yields a bound on the first quantity in \Cref{eq:hypergrad}, $\normin{v  - \hat{v}}$
    \begin{align}
        \normin{v  - \hat{v}}
        & = \normin{v - \tilde{v} + \tilde{v} - \hat{v}}
        \nonumber
        \\
        & \leq \normin{v - \tilde{v}}
        + \normin{\tilde{v} - \hat{v}}
        \nonumber
        \\
        & \leq \normin{A^{-1}A(v - \tilde{v})}
        + \normin{\tilde{v} - \hat{v}}
        \nonumber
        \\
        & \leq
        \normin{A^{-1}}_2
        \times
        \underbrace{\norm{A ( v - \tilde{v} )}}_{
            \leq \epsilon  \text{ \ref{hyp:espilon_sol}}
        }
        +
        \underbrace{\normin{\tilde{v} - \hat{v}}}_{
            \bigo (\epsilon ) \text{ (\Cref{eq:bound_v_v_tilde})}
        }
        \nonumber
        \\
        & = \bigo (\epsilon )
        \label{eq:bund_v_tilde}
        \enspace .
    \end{align}

\textit{Bound on $ \normin{B - \hat B}_2$.}
We now bound the second quantity in \Cref{eq:hypergrad} $ \normin{B - \hat B}_2$
\begin{align}
    \normin{B - \hat{B}}_2
    & \leq \normin{
        \partial_{\lambda}\prox_{\gamma g(\cdot, \lambda)}
        (\beta - \gamma \nabla f(\beta))_{\hat{S} :}
    -
        \partial_{\lambda}\prox_{\gamma g(\cdot, \lambda)}
        (\hat{\beta} -\gamma \nabla f(\hat{\beta}))_{\hat{S} :}}_2
    \nonumber
    \\
    & + \gamma \normin{\partial_z\prox_{\gamma g(\cdot, \lambda)}
        (\hat{\beta} - \gamma \nabla f(\hat{\beta}))_{\hat{S}}
        \nabla^2_{\hat{S}, \hat S^c} f(\hat{\beta})
        \hat{\jac}_{\hat{S}^c :}
    \nonumber
    \\
    & -
        \partial_z\prox_{\gamma g(\cdot, \lambda)}
        (\beta - \gamma \nabla f(\beta))_{\hat{S}}
        \nabla^2_{\hat{S}, \hat S^c} f(\beta)
        \hat{\jac}_{\hat{S}^c :}}_2
    \nonumber
    \\
    & \leq L_1 \normin{
        \beta - \gamma \nabla f(\beta)_{\hat{S} :}
        - \hat{\beta} + \gamma \nabla f(\hat{\beta})
        } \text{ using \ref{hyp:prox_lip}}
    \nonumber
    \\
    & + L_2 \normin{\hat \beta - \beta}
    \cdot \normin{\hat{\jac}_{\hat{S}^c :}}
    \text{ using \ref{hyp:prox_lip} and \Cref{ass:taylor_expansion}}
    \nonumber
    \\
    & =  \bigo(\normin{\hat \beta - \beta})
    \label{eq:bund_B_B_hat}
    \enspace .
\end{align}
Plugging \Cref{eq:bund_v_tilde,eq:bund_B_B_hat} into  \Cref{eq:hypergrad} yields the desired result: $\normin{\nabla \mathcal{L}(\lambda) - h} = \bigo(\epsilon)$.
\end{proof}

\newpage
\section{Additional Experiments}
\label{app:sec:additional_xp}
%
\subsection{Local Linear Convergence}
\Cref{fig:linear_convergence_lasso,fig:linear_convergence_logreg} are the counterparts of \Cref{fig:linear_convergence_svm} for the Lasso and sparse logistic regression.
It shows the local linear convergence of the Jacobian for the Lasso, obtained by the forward-mode differentiation of coordinate descent.
The solvers used to determine the exact solution up to machine precision are \texttt{Celer} \citep{Massias_Gramfort_Salmon18,Massias_Vaiter_Gramfort_Salmon20} for the Lasso and
\texttt{Blitz} \citep{Johnson_Guestrin15} for the sparse logistic regression.
\Cref{table:setting_linear_cv} summarizes the values of the hyperparameters $\lambda$ used in
\Cref{fig:linear_convergence_svm,fig:linear_convergence_lasso,fig:linear_convergence_logreg}.
\begin{figure*}[tb]
    \centering
    \begin{subfigure}[b]{1\textwidth}
        \centering
        \includegraphics[width=0.4\linewidth]{linear_convergence_lasso_legend}
        \includegraphics[width=\figsize\linewidth]{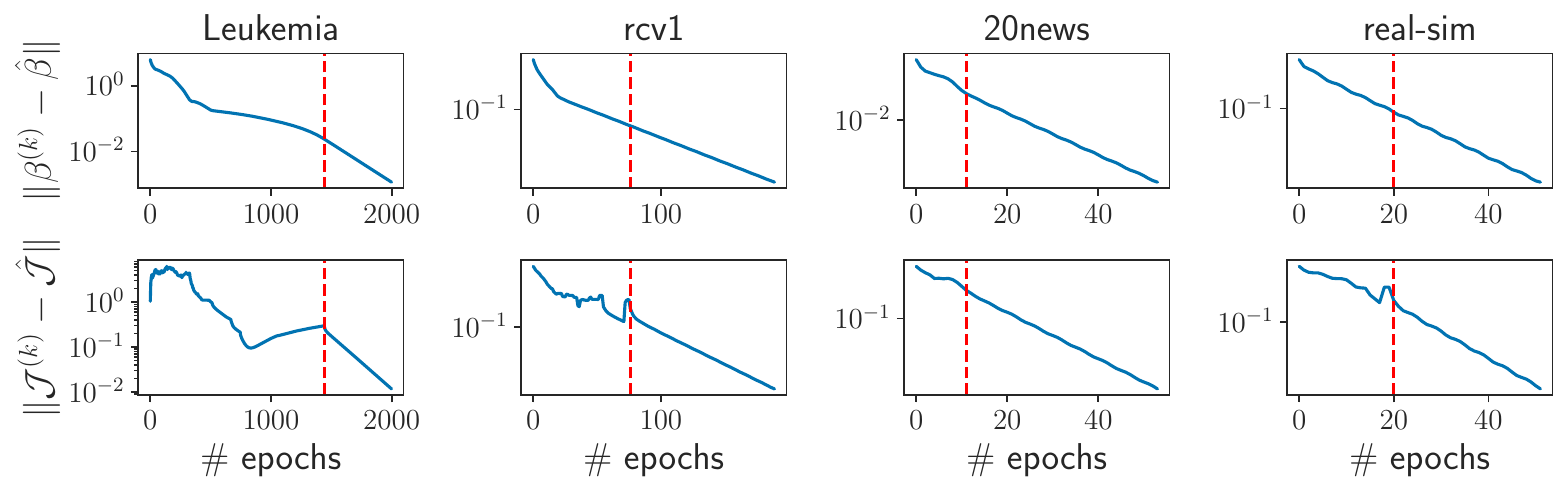}
    \end{subfigure}
    \caption{
        \textbf{Local linear convergence of the Jacobian for the Lasso.}
    Distance to optimum for the coefficients $\beta$ (top) and the Jacobian
    $\jac$ (bottom) of the \forward differentiation of proximal coordinate descent (\Cref{alg:forward_pcd}) on multiple data sets.}
    \label{fig:linear_convergence_lasso}
\end{figure*}
\begin{figure*}[tb]
    \centering
    \begin{subfigure}[b]{1\textwidth}
        \centering
        \includegraphics[width=0.4\linewidth]{linear_convergence_lasso_legend}
        \includegraphics[width=\figsize\linewidth]{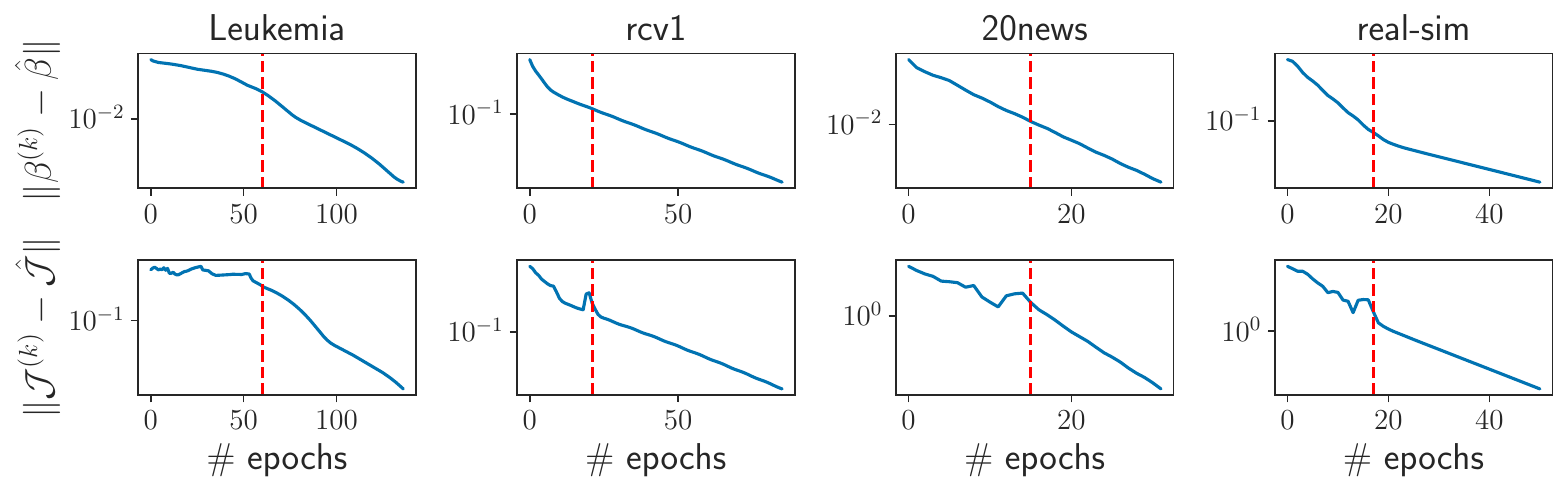}
    \end{subfigure}
    \caption{
        \textbf{Local linear convergence of the Jacobian for sparse logistic regression.}
    Distance to optimum for the coefficients $\beta$ (top) and the Jacobian $\jac$ (bottom) of the \forward differentiation of proximal coordinate descent (\Cref{alg:forward_pcd}) on multiple data sets.}
    \label{fig:linear_convergence_logreg}
\end{figure*}
\begin{table}
    \caption{Data set characteristics and regularization parameters used in \Cref{fig:linear_convergence_lasso,fig:linear_convergence_logreg,fig:linear_convergence_svm}.}
    \begin{tabular}{|c|c|c|c|c|}
        \hline
        Data sets
        & \emph{leukemia}
        & \emph{rcv1}
        & \emph{news20}
        & \emph{real-sim} \\
        \# samples
        & $n=38$
        & $n=\num{20242}$
        & $n=\num{19996}$
        & $n=\num{72309}$ \\
        \# features
        & $p=\num{7129}$
        & $p=\num{19959}$
        & $p=\num{632982}$
        & $p=\num{20958}$ \\
        \hline
        \rule{0pt}{2.5ex}
        Lasso
        & $e^\lambda = 0.01 \, e^{ \lambda_{\text{max}} }$
        & $e^\lambda = 0.075 \, e^{ \lambda_{\text{max}} }$
        & $e^\lambda = 0.3 \, e^{ \lambda_{\text{max}} }$
        & $e^\lambda = 0.1 \, e^{ \lambda_{\text{max}} }$ \\
        Logistic regression
        & $e^\lambda = 0.1 \, e^{ \lambda_{\text{max}} }$
        & $e^\lambda = 0.25 \, e^{ \lambda_{\text{max}} }$
        & $e^\lambda = 0.8 \, e^{ \lambda_{\text{max}} }$
        & $e^\lambda = 0.15 \, e^{ \lambda_{\text{max}} }$\\
        SVM
        & $e^\lambda = 10^{-5} $
        & $e^\lambda = 3 \times 10^{-2}$
        & $e^\lambda = 10^{-3}$
        & $e^\lambda = 5 \times 10^{-2}$ \\
        \hline
      \end{tabular}
      \label{table:setting_linear_cv}
    \end{table}

\subsection{Hypergradient Computation Time}
    The experimental setting for \Cref{fig:hypergradient_cvxpy_wlasso} is the same as for \Cref{fig:hypergradient_cvxpy}, but with a
    weighted Lasso (\ie $p$ hyperparameters to optimize) as inner problem.
    It represents the time needed by the different algorithms to compute a single hypergradient as a function of the number of features.
    The regularization amounts were chosen uniformly at random in the interval $[0, e^\lambda_{\max}]$
    and each point represents $10$ repetitions.
    \Cref{fig:hypergradient_cvxpy_wlasso} shows that when the number of hyperparameters is large the \implicitfull outperforms the \backward, ans the \backward outperforms the \forward, by one or more orders of magnitude.
    This corroborates the complexities summarized in \Cref{tab:summary_costs}.
    \begin{figure*}[tb]
        \centering
        \begin{subfigure}[b]{1\textwidth}
            \centering
            \includegraphics[width=0.55\linewidth]{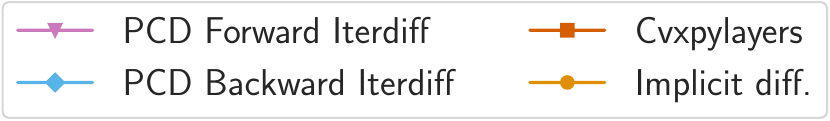}
            \includegraphics[width=0.5\linewidth]{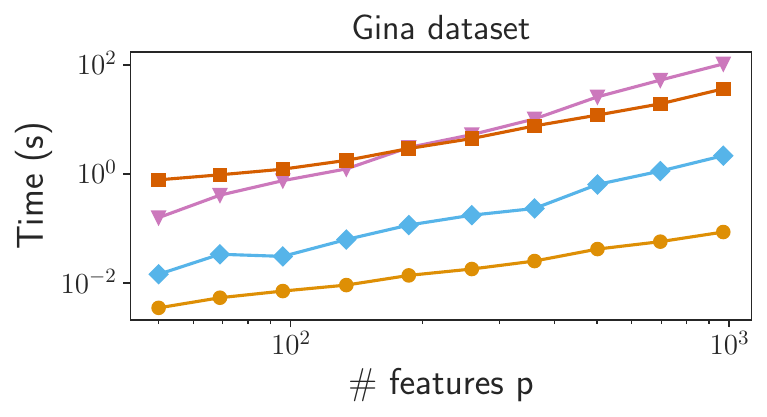}
        \end{subfigure}
        \caption{
            \textbf{Weighted Lasso with hold-out criterion.}
            Time to compute a single hypergradient as a function of the number of features on the \emph{gina} data set.
            The regularization parameters have been chosen uniformly at random in the range $[0, e^\lambda_{\max}]$.
        }
        \label{fig:hypergradient_cvxpy_wlasso}
    \end{figure*}

\end{document}